\documentclass[journal,10pt]{IEEEtran}
\usepackage{cite}

\ifCLASSINFOpdf
   \usepackage[pdftex]{graphicx}
\else
\fi
\usepackage[cmex10]{amsmath}
\usepackage{amsfonts,amssymb,amsthm}
\usepackage{algorithmic}

\usepackage{array}
\newsavebox{\ieeealgbox}
\newenvironment{boxedalgorithmic}
  {\begin{lrbox}{\ieeealgbox}
   \begin{minipage}{\dimexpr\columnwidth-2\fboxsep-2\fboxrule}
   \begin{algorithmic}}
  {\end{algorithmic}
   \end{minipage}
   \end{lrbox}\noindent\fbox{\usebox{\ieeealgbox}}}

   \usepackage{calc}
   
\usepackage{bbding}
\usepackage{comment}
\usepackage{mdwlist}
\usepackage{enumerate}
\usepackage{units}
\usepackage{multirow}
\usepackage{tabularx}
\newcolumntype{Y}{>{\centering\arraybackslash}X}
\usepackage{booktabs}
\usepackage{arydshln}
\usepackage{wasysym}
\usepackage{color}
\usepackage{mdwlist}
\usepackage{enumerate}
\usepackage{hyperref}\hypersetup{colorlinks=false,citecolor=black,filecolor=black,linkcolor=black,urlcolor=black}

\newtheorem{theorem}{Theorem}[section]
\newtheorem{lemma}[theorem]{Lemma}
\newtheorem{proposition}[theorem]{Proposition}

\newtheorem{definition}[theorem]{Definition}

\newenvironment{remark}[1][Remark]{\begin{trivlist}
\item[\hskip \labelsep {\bfseries #1}]}{\end{trivlist}}

\newcommand\cmt[1]{{\color{red} #1}}

\newcommand\gcmt[1]{{\color{green} #1}}

\newcommand\modif[1]{{\color{black} #1}}

\begin{document}
%
\title{Flexible Multi-layer Sparse Approximations of Matrices and Applications}
%
%
%

\author{Luc~Le Magoarou,
        R\'emi~Gribonval,~\IEEEmembership{Fellow,~IEEE}
\thanks{Luc~Le Magoarou (\href{mailto:luc.le-magoarou@inria.fr}{\texttt{luc.le-magoarou@inria.fr}}) and R\'emi Gribonval (\href{mailto:remi.gribonval@inria.fr}{\texttt{remi.gribonval@inria.fr}}) are both with Inria, Rennes, France, PANAMA team. This work was supported in part by the European Research Council, PLEASE project (ERC-StG- 2011-277906). Parts of this work have been presented at the conferences ICASSP 2015  \cite{Lemagoarou2014} and EUSIPCO 2015 \cite{Lemagoarou2015}. Copyright (c) 2014 IEEE. Personal use of this material is permitted. However, permission to use this material for any other purposes must be obtained from the IEEE by sending a request to \href{mailto:pubs-permissions@ieee.org}{pubs-permissions@ieee.org}.
}
}

%
%

\markboth{IEEE Journal of selected topics in signal processing}%
{Shell \MakeLowercase{\textit{et al.}}: Bare Demo of IEEEtran.cls for Journals}
%



\maketitle

\begin{abstract}
The computational cost of many signal processing and machine learning techniques is often dominated by the cost of applying certain linear operators to high-dimensional vectors. This paper introduces an algorithm aimed at reducing the complexity of applying linear operators in high dimension by approximately factorizing the corresponding matrix into few sparse factors. 
The approach relies on recent advances in non-convex optimization. It is first explained and analyzed in details and then demonstrated experimentally on various problems including dictionary learning for image denoising, and the approximation of large matrices arising in inverse problems.
\end{abstract}

\begin{IEEEkeywords}
Sparse representations, fast algorithms, dictionary learning, low
complexity, image denoising, inverse problems.
\end{IEEEkeywords}

%
\IEEEpeerreviewmaketitle

\section{Introduction}
%
%
%
%
\IEEEPARstart{S}{parsity} has been at the heart of a plethora of signal processing and data analysis techniques over the last two decades. These techniques usually impose that the objects of interest be sparse in a certain domain. They owe their success to the fact that sparse objects are easier to manipulate and more prone to interpretation than dense ones especially in high dimension. However, to efficiently manipulate high-dimensional data, it is not sufficient to rely on sparse objects: efficient operators are also needed to manipulate these objects. 
 
The $n$-dimensional Discrete Fourier Transform (DFT) is certainly the most well known linear operator with an efficient implementation: the Fast Fourier Transform (FFT) \cite{CooleyTukey1965}, allows to apply the operator in $\mathcal{O}(n\log n)$ arithmetic operations instead of $\mathcal{O}(n^2)$ in its dense form. Similar complexity savings have been achieved for other widely used operators such as the Hadamard transform \cite{Shanks1969}, the Discrete Cosine Transform (DCT) \cite{Wen1977} or the Discrete Wavelet Transform (DWT)\cite{Mallat1989}. For all these fast linear transforms, the matrix $\mathbf{A}$ corresponding to the dense form of the operator admits a {\em multi-layer sparse} expression,
\begin{equation}
\mathbf{A} = \prod_{j=1}^J\mathbf{S}_j,
\label{eq:spop}
\end{equation} 
corresponding to a multi-layer factorization\footnote{The product being taken from right to left: $\prod_{j=1}^J \mathbf{S}_j = \mathbf{S}_J \cdots \mathbf{S}_1$}
 into a small number $J$ of sparse factors $\mathbf{S}_j$. 
 Following the definition of a linear algorithm given in \cite{Morgenstern1975}, this multi-layer sparse factorization is actually the natural representation of any fast linear transform.

\begin{figure}[htbp]
\includegraphics[width=\columnwidth]{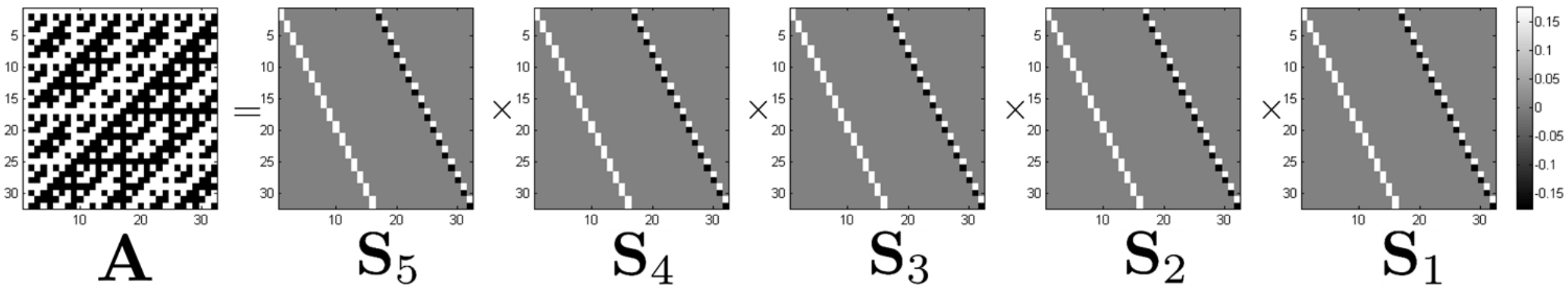}
\caption{The Hadamard matrix of size $n\times n$ with $n=32$ (left) and its factorization. The matrix is totally dense so that the naive storage and multiplication cost $\mathcal{O}(n^2=1024)$. On the other hand, we show the factorization of the matrix into $\log_2(n) = 5$ factors, each having $2n = 64$ non-zero entries, so that the storage and multiplication in the factorized form cost \mbox{$\mathcal{O}(2n\log_2(n)=320)$.}}
\vspace{-0.2cm}
\label{fig:facthadamard}
\end{figure}

For example each step of the butterfly radix-2 FFT can be seen as the multiplication by a sparse matrix having only two non-zero entries per row and per column. This fact is further illustrated in the case of the Hadamard transform on Figure~\ref{fig:facthadamard}. For other examples, see e.g. \cite[Appendix A]{Lemagoarou2014}.

Inspired by these widely used transforms, our objective is to find approximations of operators of interest encountered in concrete applications, as products of sparse matrices as in \eqref{eq:spop}. Such approximations will be called \emph{Flexible Approximate MUlti-layer Sparse Transforms (FA$\mu$ST)}.

As a primary example of potential application of such approximations, consider \emph{linear inverse problems}, where data and model parameters are linked through a linear operator. 
State of the art algorithms addressing such problems with sparse regularization \cite{Mallat1993,Daubechies2004,Tropp2007,Blumensath2008,Beck2009} are known to heavily rely on matrix-vector products involving both this operator and its adjoint. As illustrated in Section~\ref{sec:MEG} on a biomedical inverse problem, replacing the operator by an accurate FA$\mu$ST has the potential to substantially accelerate these methods. 

To choose a regularizer for inverse problems,  \emph{dictionary learning} is a common method used to learn the domain in which some training data admits a sparse representation \cite{Rubinstein2010}. Its applicability is however also somewhat limited by the need to compute many matrix-vector products involving the learned dictionary and its adjoint, which are in general dense matrices. We will see that recent approaches to learn fast dictionaries\cite{Rubinstein2010a,Chabiron2013} can be seen as special cases of the {\em FA$\mu$ST dictionary learning} approach developed in Section~\ref{sec:diclearn}, where the learned dictionaries are constrained to be FA$\mu$ST.
 
Beyond the above considered examples, any task where it is required to apply a linear operator in high dimension would obviously benefit from a FA$\mu$ST corresponding to the considered operator.   
For example,  in the emerging area of \emph{signal processing on graphs} \cite{Shuman2013}, novel definitions of usual operators such as the Fourier or wavelet transforms have been introduced. They have no known general sparse forms, and consequently no associated fast algorithms. Finding multi-layer sparse approximations of these usual operators on graphs would certainly boost the dissemination and impact of graph signal processing techniques.

\noindent{\bf Objective.} 
The quest for multi-layer sparse approximations of large linear operators, which is the core objective of this paper, actually amounts to a matrix factorization problem, where the matrix $\mathbf{A}$ corresponding to the dense form of the operator is to be decomposed into the product of sparse factors $\mathbf{S}_j$, so as to satisfy an approximate form of \eqref{eq:spop}.

 \noindent{\bf Contributions.} 
 This paper substantially extends
the preliminary work started in \cite{Lemagoarou2014}, \cite{Lemagoarou2015} and \cite{Lemagoarou2015a}, both on the theoretical and experimental sides, with the following contributions:
\begin{itemize}
\item A general framework for multi-layer sparse approximation (MSA) is introduced, that allows to incorporate various constraints on the sought sparse form;
\item Recent advances in non-convex optimization \cite{Bolte2013} are exploited to tackle the resulting non-convex optimization problem with local convergence guarantees; 
\item A heuristic hierarchical factorization algorithm leveraging these optimization techniques is proposed, that achieves factorizations empirically stable to initialization;
\item The versatility of the framework is illustrated with extensive experiments on two showcase applications, linear inverse problems and dictionary learning, demonstrating its practical benefits.
\end{itemize}

The remaining of the paper is organized as follows. The problem is formulated, related to prior art and the expected benefits of FA$\mu$STs are systematically explained in Section~\ref{sec:formul}. A general optimization framework for the induced matrix factorization problem is introduced in Section~\ref{sec:optim} and Section~\ref{ssec:practical}, and as a first illustration we demonstrate that it is possible to {\em reverse-engineer} the Hadamard transform. Several applications and experiments on various tasks, illustrating the versatility of the proposed approach are performed in sections~\ref{sec:MEG} and~\ref{sec:diclearn}.

\section{Problem formulation}
\label{sec:formul}
\noindent{\bf Notation.} Throughout this paper, matrices are denoted by bold upper-case letters: $\mathbf{A}$; vectors  by bold lower-case letters: $\mathbf{a}$; the $i$th column of a matrix $\mathbf{A}$ by: $\mathbf{a}_i$; 
and sets by calligraphic symbols: $\mathcal{A}$. The standard vectorization operator is denoted by $\text{vec}(\cdot)$. The $\ell_0$-norm is denoted by $\left\Vert\cdot\right\Vert_0$ (it counts the number of non-zero entries), $\left\Vert\cdot\right\Vert_F$ denotes the Frobenius norm, and $\left\Vert\cdot\right\Vert_{2}$ the spectral norm. By abuse of notations, $\|\mathbf{A}\|_{0} = \|\text{vec}(\mathbf{A})\|_{0}$. The identity matrix is denoted $\mathbf{Id}$.

\subsection{Objective}
\label{ssec:objective}
The goal of this paper is to introduce a method to get a FA$\mu$ST associated to an operator of interest. Consider a linear operator corresponding to the matrix $\mathbf{A} \in \mathbb{R}^{m\times n}$. The objective is to find sparse factors $\mathbf{S}_j \in \mathbb{R}^{a_{j+1}\times a_j}, j \in \{1\ldots J\}$ with $a_1 = n$ and $a_{J+1} = m$ such that $\mathbf{A}\approx \prod_{j=1}^J \mathbf{S}_j$. This naturally leads to an optimization problem of the form:
\begin{equation}
\begin{array}{c}
\underset{\mathbf{S}_1, \ldots,\mathbf{S}_{J}}{\text{Minimize }}\quad   \underbrace{\big\Vert\mathbf{A} - \prod\limits_{j=1}^{J}\mathbf{S}_j\big\Vert^2}_{\text{Data fidelity}} + \underbrace{\sum\limits_{j=1}^{J}g_j(\mathbf{S}_j)}_{\text{Sparsity-inducing penalty}},
\end{array}
\label{eq:verygeneralproblem}
\end{equation}
to trade-off data fidelity and sparsity of the factors.

\subsection{Expected benefits of FA$\mu$STs}
\label{ssec:advantages}
A multi-layer sparse approximation of an operator $\mathbf{A}$ brings several benefits, provided the {\em relative complexity} of the factorized form is small with respect to the dimensions of $\mathbf{A}$. For the sake of conciseness, let us introduce $s_j = \left\Vert \mathbf{S}_j \right\Vert_0$ the total amount of non-zero entries in the $j$th factor, and $s_{tot} = \sum_{j=1}^J s_j$ the total number of non-zero entries in the whole factorization. 
\begin{definition}
The Relative Complexity (abbreviated RC) is the ratio between the total number of non-zero entries in the FA$\mu$ST and the number of non-zero entries of $\mathbf{A}$: 
\begin{equation}
\text{RC} := \frac{s_{tot}}{\left\Vert \mathbf{A} \right\Vert_0} .
\label{RC} 
\end{equation} 
It is also interesting to introduce the Relative Complexity Gain (RCG), which is simply the inverse of the Relative Complexity ($\text{RCG} = 1/\text{RC}$).
\end{definition}
The aforementioned condition for the factorized form to be beneficial writes: $\text{RC}\ll 1$ or equivalently $\text{RCG}\gg 1$.

FA$\mu$STs reduce computational costs in all aspects of their manipulation, namely a lower \emph{Storage cost}, a higher \emph{Speed of multiplication} and an improved \emph{Statistical significance}.

\subsubsection{Storage cost} 
 Using the Coordinate list (COO) storage paradigm \cite{scipy}, one can store a FA$\mu$ST using $\mathcal{O}(s_{tot})$ floats and integers. Indeed each non-zero entry (float) in the factorization can be located using three integers (one for the factor, one for the row and one for the column), which makes $s_{tot}$ floats and $3s_{tot}$ integers to store. One needs also $J+1$ supplementary integers to denote the size of the factors $a_1$ to $a_{J+1}$. In summary the storage gain is of the order of RCG.
\subsubsection{Speed of multiplication}
Applying the FA$\mu$ST or its transpose to a vector $\mathbf{v} \in \mathbb{R}^n$ can be easily seen to require at most $\mathcal{O}(s_{tot})$ floating point operations (flops), instead of $\mathcal{O}(mn)$ for a classical dense operator of same dimension, so the computational gain is, like the storage gain, of the order of RCG. 

\subsubsection{Statistical significance}

 Another interesting though less obvious benefit of FA$\mu$STs over dense operators arises when the operator has to be estimated from training data as in dictionary learning. In this case the reduced number of parameters to learn --$\mathcal{O}(s_{tot})$ compared to $\mathcal{O}(mn)$ for dense operators-- leads to better statistical properties. More specifically, the sample complexity is reduced \cite{Gribonval2015}, and better generalization properties are expected. The sample complexity gain is of the order of RCG, as will be shown in the case of dictionary learning. The impact of these gains will be illustrated experimentally in section~\ref{sec:diclearn} on image denoising with learned dictionaries.

\subsection{Related work}
\label{ssec:relwork}

Similar matrix factorization problems have been studied in several domains. Some are very classical tools from numerical linear algebra, such as the truncated SVD, while other emerged more recently in signal processing and machine learning.

\subsubsection{The truncated SVD}
To reduce the computational complexity of a linear operator, the most classical approach is perhaps to compute a low-rank approximations with the truncated SVD. Figure~\ref{fig:compSVD} compares the approximation-complexity trade-offs achieved via a low-rank approximation (truncated SVD) and via a multi-layer sparse approximation, on a forward operator associated to an MEG inverse problem. The truncated SVD and four FA$\mu$STs computed using  different configurations (more details in Section~\ref{sec:MEG} - Figure~\ref{fig:MEG}) are compared in terms of relative operator norm error: $\small \Vert\mathbf{A}- \hat{\mathbf{A}}\Vert_2/\Vert\mathbf{A}\Vert_2$. It is readily observed that the FA$\mu$STs achieve significantly better complexity/error trade-offs.

 \begin{figure}[h]
    \centering
        \includegraphics[width=1\columnwidth]{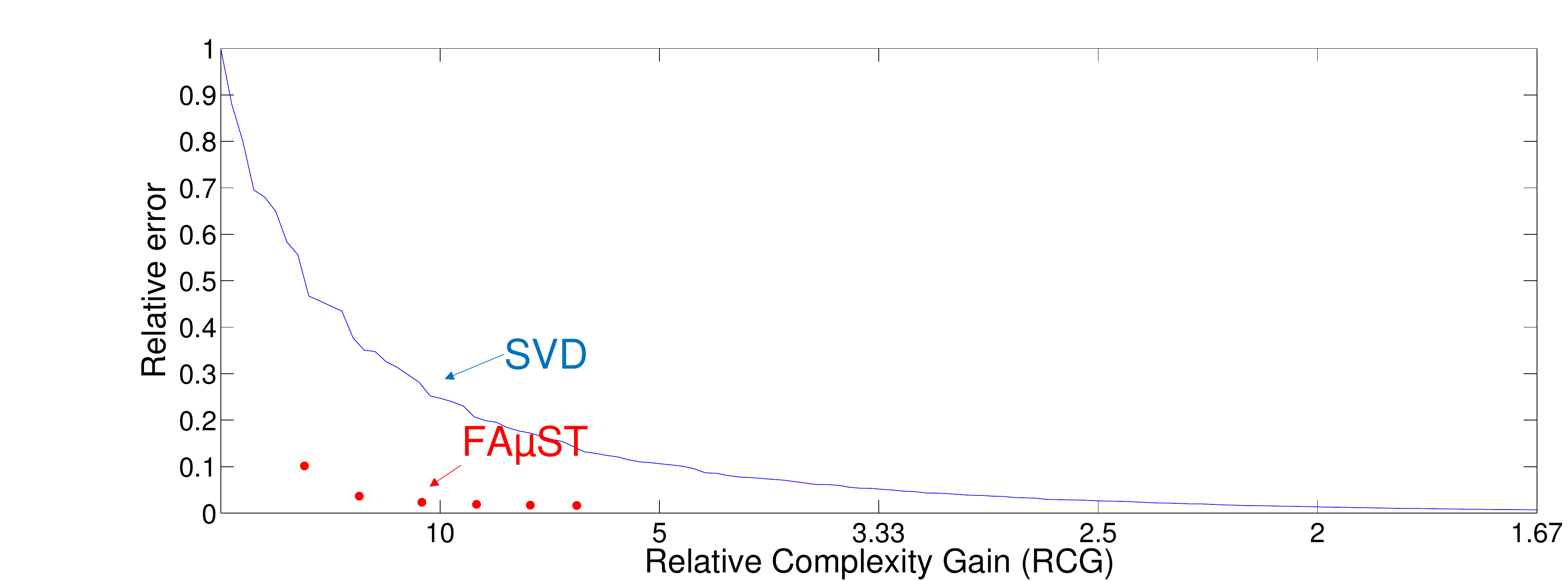}
        \vspace{-0.6cm}
    \caption{Comparison between four multi-layer sparse approximations correponding to different configurations and the truncated SVD. A $204\times 8193$ matrix associated to an MEG inverse problem is used for this example.}  
    \vspace{-0.1cm}
    \label{fig:compSVD}
\end{figure}

\subsubsection{Local low-rank approximations}
\label{sssec:loclow} 
Given the limitations of global low-rank approximation by truncated SVD illustrated in Figure~\ref{fig:compSVD}, the numerical linear algebra community has developed {\em local} approximations of operators by low-rank patches. \modif{This general operator compression paradigm encompasses several methods introduced in the last three decades, including the Fast Multipole Method (FMM) \cite{Rokhlin1985}, H-matrices \cite{Hackbusch1999} and others \cite{Candes2007}.  All the difficulty of these methods resides in the choice of the patches, which is done according to the regularity of the subsumed continuous kernel.} This can be seen as approximating the operator by a FA$\mu$ST, where the support of each factor is determined analytically. The approach proposed in this paper is data-driven rather than analytic. 

\modif{
\subsubsection{Wavelet-based compression}
\label{sssec:wavcomp}
 This operator compression paradigm, introduced in \cite{Beylkin1991}, is based on the use of orthogonal wavelets in the column domain (associated to a matrix $\boldsymbol{\Phi}_1$), and in the row domain (associated to a matrix $\boldsymbol{\Phi}_2$) of the matrix $\mathbf{A}$ to approximate. By orthogonality of the matrices $\boldsymbol{\Phi}_1$ and $\boldsymbol{\Phi}_2$, we have $\mathbf{A} = \boldsymbol{\Phi}_1\boldsymbol{\Phi}_1^T \mathbf{A} \boldsymbol{\Phi}_2\boldsymbol{\Phi}_2^T$. The wavelet-based compression scheme relies on the fact that $\mathbf{B}\triangleq \boldsymbol{\Phi}_1^T \mathbf{A} \boldsymbol{\Phi}_2$ is compressible (provided appropriate wavelets are chosen relative to the subsumed kernel). This implies $\mathbf{B}\approx \hat{\mathbf{B}}$ where $\hat{\mathbf{B}}$ is sparse so that $\mathbf{A} \approx \hat{\mathbf{A}} =  \boldsymbol{\Phi}_1\hat{\mathbf{B}}\boldsymbol{\Phi}_2^T$. Fast multiplication by $\hat{\mathbf{A}}$ is possible as soon as $\hat{\mathbf{B}}$ is sparse enough and the wavelet transforms $\boldsymbol{\Phi}_1$ and $\boldsymbol{\Phi}_2$ have fast implementations. This can be seen as approximating $\mathbf{A}$ by a FA$\mu$ST.}

\subsubsection{Dictionary learning} 
Given a collection of training vectors $\mathbf{y}_{\ell}$, $1 \leq \ell \leq L$ gathered as the columns of a matrix $\mathbf{Y}$, the objective of dictionary learning \cite{Rubinstein2010,Tosic2011} is to approximate $\mathbf{Y}$ by the product of a dictionary $\mathbf{D}$ and a coefficients matrix $\boldsymbol{\Gamma}$ with sparse columns, $\mathbf{Y} \approx \mathbf{D}\boldsymbol{\Gamma}$. 

To learn dictionaries with improved computational efficiency, two main lines of work have begun to explore approaches related to multi-layer sparse approximation. In \cite{Rubinstein2010a}, the authors propose the sparse-KSVD algorithm (KSVDS) to learn a dictionary whose atoms are sparse linear combinations of atoms of a so-called \emph{base dictionary} $\mathbf{D}_{\textrm{base}}$. The base dictionary should be associated with a fast algorithm (in practice, this means that it is a FA$\mu$ST) so that the whole learned dictionary is itself a FA$\mu$ST. It can be seen as having the $J-1$ leftmost factors fixed in~\eqref{eq:verygeneralproblem}, their product being precisely $\mathbf{D}_{\textrm{base}}$, while the first factor $\mathbf{S}_{1}$ is the sparse representation of the dictionary over the base dictionary, i.e., $\mathbf{D} = \mathbf{D}_{\textrm{base}} \mathbf{S}_1$. 

A limitation of the sparse-KSVD formulation is that the learned dictionary is highly biased toward the base dictionary, which decreases adaptability to the training data. In \cite{Chabiron2013}, the authors propose to learn a dictionary in which each atom is the composition of several circular convolutions using sparse kernels with known supports, so that the dictionary is a sparse operator that is fast to manipulate. This problem can be seen as~\eqref{eq:verygeneralproblem}, with the penalties $g_j$s associated to the $J-1$ leftmost factors imposing sparse circulant matrices with prescribed supports. This formulation is powerful, as demonstrated in \cite{Chabiron2013}, but limited in nature to the case where the dictionary is well approximated by a product of sparse circulant matrices, and requires knowledge of the supports of the sparse factors.

\subsubsection{Inverse problems} 
In the context of sparse regularization of linear inverse problems, one is given a signal $\mathbf{y}$ and a measurement matrix $\mathbf{M}$ and wishes to compute a sparse code $\boldsymbol{\gamma}$ such that $\mathbf{y} \approx \mathbf{M} \boldsymbol{\gamma}$, see e.g. \cite{Bruckstein2009}. Most modern sparse solvers rely on some form of iterative thresholding and heavily rely on matrix-vector products with the measurement matrix and its transpose.
 Imposing --and adjusting-- a FA$\mu$ST structure to approximate these matrices as proposed here has the potential to further accelerate these methods through fast matrix-vector multiplications. This is also likely to bring additional speedups to recent approaches accelerating iterative sparse solvers through learning \cite{Gregor2010,Sprechmann2012,Makhzani2013}.

\subsubsection{Statistics -- factor analysis} 
A related problem is to approximately diagonalize a covariance matrix by a unitary matrix in factorized form~\eqref{eq:spop}, which can be addressed greedily \cite{Lee2008,Cao2011}  using a fixed number of elementary Givens rotations. Here we consider a richer family of sparse factors and leverage recent non-convex optimization techniques. 

\subsubsection{Machine learning} 
Similar models were explored with various points of view in machine learning. For example, sparse multi-factor NMF \cite{Lyu2013} can be seen as solving problem~\eqref{eq:verygeneralproblem} with the Kullback-Leibler divergence as data fidelity term and all factors $\mathbf{S}_j$s constrained to be non-negative. Optimization relies on multiplicative updates, while the approach proposed here relies on proximal iterations. 

\subsubsection{Deep learning} 
In the context of deep neural networks, identifiability guarantees on the network structure have been established with a generative model where consecutive network layers are sparsely connected at random, and non-linearities are neglected \cite{Neyshabur2013,Arora2013}. The network structure in these studies matches the factorized structure~\eqref{eq:spop}, with each of the leftmost factors representing a layer of the network and the last one being its input. Apart from its hierarchical flavor, the identification algorithm in~ \cite{Neyshabur2013,Arora2013} has little in common with the proximal method proposed here.

\subsubsection{Signal processing on graphs} 
Similar matrix factorizations problems arise in this domain, with the objective of defining wavelets on graphs. First, in \cite{Kondor2014} the authors propose to approximately diagonalize part of the graph Laplacian operator using elementary rotations. More precisely, the basis in which the Laplacian is expressed is greedily changed, requiring that at each step the change is made by a sparse elementary rotation (so that the wavelet transform is multi-layer sparse), and variables are decorrelated. The Laplacian ends up being diagonal in all the dimensions corresponding to the wavelets and dense in a small part corresponding to the scaling function (the algorithm ends up being very similar to the one proposed in \cite{Cao2011}). Second, in \cite{Rustamov2013} the authors propose to define data adaptive wavelets by factorizing some training data matrix made of signals on the graph of interest. The constraint they impose to the wavelet operator results in a multi-layer sparse structure, where each sparse factor is further constrained to be a lifted wavelet building block. The optimization algorithm they propose relies on deep learning techniques, more precisely layer-wise training of stacked auto-encoders \cite{Bengio2007}.

\section{Optimization framework}
\label{sec:optim}

\subsection{Objective function}
\label{ssec:cobjfunc}
In this paper, 
the penalties $g_j(\cdot)$ 
appearing in the general form of the optimization problem~\eqref{eq:verygeneralproblem} are chosen as indicator functions $\delta_{\mathcal{E}_j}(\cdot)$ of constraint sets of interest $\mathcal{E}_j$. To avoid the scaling ambiguities arising naturally when the constraint sets are (positively) homogeneous\footnote{This is the case of many standard constraint sets. In particular all unions of subspaces, such as the sets of sparse or low-rank matrices, are homogeneous.}, it is common \cite{Chabiron2013,Lyu2013} to normalize the factors and introduce a multiplicative scalar $\lambda$ in the data fidelity term. For that, let us introduce the sets of normalized factors $\mathcal{N}_j = \{\mathbf{S} \in \mathbb{R}^{a_{j+1} \times a_{j}} :  \left\Vert \mathbf{S} \right\Vert_F = 1\}$, and impose the following form for the constraints sets: $\mathcal{E}_j = \mathcal{N}_j \cap \mathcal{S}_j$, where $\mathcal{S}_j$ imposes sparsity explicitly or implicitly.  This results in the following optimization problem:
\begin{equation}
\begin{array}{rl}
\underset{\lambda,\mathbf{S}_1, \ldots,\mathbf{S}_{J}}{\text{Minimize }}\quad \Psi(\mathbf{S}_1, \ldots ,\mathbf{S}_{J},\lambda) :=& \frac{1}{2} \Big\Vert\mathbf{A} - \lambda\prod\limits_{j=1}^{J}\mathbf{S}_j\Big\Vert_F^2\\ & + \sum\limits_{j=1}^{J}\delta_{\mathcal{E}_j}(\mathbf{S}_j).
\end{array}
\label{eq:actualproblemscal}
\end{equation}



As will be made clear below, the used minimization algorithm relies on projections onto the constraint sets $\mathcal{E}_j$: the choice of the ``sparsity-inducing'' part of the constraint sets $\mathcal{S}_j$ is quite free provided that the projection operator onto these sets is known. 

A comon choice is to limit the total number of non-zero entries in the factors to $s_j$. The constraint sets then take the form $\mathcal{E}_j = \{\mathbf{S} \in \mathbb{R}^{a_{j+1} \times a_{j}} :  \left\Vert \mathbf{S} \right\Vert_0 \leq s_j,  \left\Vert \mathbf{S} \right\Vert_F = 1\}$. Another natural choice is to limit to $k_j$ the number of non-zero entries per row or column in the factors, which gives for example in the case of the columns $\mathcal{E}_j = \{\mathbf{S} \in \mathbb{R}^{a_{j+1} \times a_{j}} :  \left\Vert \mathbf{s}_i \right\Vert_0 \leq k_j \, \forall i,  \left\Vert \mathbf{S} \right\Vert_F = 1\}$. Other possible constraint sets can be chosen to further impose non-negativity, a circulant structure, a prescribed support, etc., see for example \cite{Chabiron2013}. 

Besides the few examples given above, many more choices of penalties beyond indicator functions of constraint sets can be envisioned in the algorithmic framework described below. Their choice is merely driven by the application of interest, as long as they are endowed with easy to compute projections onto the constraint sets (in fact, efficient proximal operators), and satisfy some technical assumptions (detailed below) that are very often met in practice. We leave the full exploration of this rich field and its possible applications to further work.

\subsection{Algorithm overview}
\label{ssec:algoverview}
Problem~\eqref{eq:actualproblemscal} is highly non-convex, and the sparsity-inducing penalties are typically non-smooth. Stemming on recent advances in non-convex optimization, it is nevertheless possible to propose an algorithm with convergence guarantees to a stationary point of the problem. In \cite{Bolte2013}, the authors consider cost functions depending on $N$ blocks of variables of the form:
\begin{equation}
\Phi(\mathbf{x}_1,\ldots,\mathbf{x}_N) := H(\mathbf{x}_1,\ldots,\mathbf{x}_N) + \sum\limits_{j=1}^Nf_j(\mathbf{x}_j) ,
\label{eq:PALMobjective}
\end{equation}
where the function $H$ is smooth, and the penalties $f_j$s are proper and lower semi-continuous (the exact assumptions are given below). It is to be stressed that \emph{no convexity} of any kind is assumed. Here, we assume for simplicity that the penalties $f_j$s are indicator functions of constraint sets $\mathcal{T}_j$. To handle this objective function, the authors propose an algorithm called Proximal Alternating Linearized Minimization (PALM)\cite{Bolte2013}, that updates alternatively each block of variable by a proximal (or projected in our case) gradient step. The structure of the PALM algorithm is given in Figure~\ref{fig:algo_summary}, where $P_{\mathcal{T}_j}(\cdot)$ is the projection operator onto the set $\mathcal{T}_j$ and $c^i_j$ defines the step size and depends on the Lipschitz constant of the gradient of $H$.

\begin{figure}[htbp]
PALM (summary) \\
\begin{boxedalgorithmic}[1]  
\FOR{$i \in \{1 \cdots Niter\} $} 
\FOR{$j \in \{1 \cdots N\} $}
\STATE Set $\mathbf{x}_j^{i+1} = P_{\mathcal{T}_j}\Big(\mathbf{x}_j^{i} - \frac{1}{c^i_j}\nabla_{\mathbf{x}_j}H\big(\mathbf{x}_1^{i+1}\scriptsize{\ldots}\mathbf{x}_j^{i}\scriptsize{\ldots}\mathbf{x}_{N}^{i}\big)\Big)$
\ENDFOR
\ENDFOR
\end{boxedalgorithmic}
\caption{PALM algorithm (summary).}
\vspace{-0.2cm}
\label{fig:algo_summary}
\end{figure}

The following conditions are sufficient (not necessary) to ensure that each bounded sequence generated by PALM converges to a stationary point of its objective \cite[Theorem 3.1]{Bolte2013} (the sequence converges, which implies convergence of the value of the cost function):
\begin{enumerate}[(i)]
\item The $f_j$s are proper and lower semi-continuous.
\item $H$ is smooth.
\item $\Phi$ is semi-algebraic \cite[Definition 5.1]{Bolte2013}.
\item $\nabla_{\mathbf{x}_j}H$ is globally Lipschitz for all $j$, with Lipschitz moduli $L_j(\mathbf{x}_1\scriptsize{\ldots}\mathbf{x}_{j-1},\mathbf{x}_{j+1}\scriptsize{\ldots}\mathbf{x}_{N})$.
\item $\forall i,j$, $c^i_j>L_j(\mathbf{x}_1^{i+1}\scriptsize{\ldots}\mathbf{x}_{j-1}^{i+1},\mathbf{x}_{j+1}^{i}\scriptsize{\ldots}\mathbf{x}_{N}^{i})$ (the inequality need not be strict for convex $f_j$).
\end{enumerate}

\subsection{Algorithm details}
\label{ssec:algdetails}
 PALM can be instantiated for the purpose of handling the objective of \eqref{eq:actualproblemscal}. It is quite straightforward to see that there is a match between \eqref{eq:actualproblemscal} and \eqref{eq:PALMobjective} by taking $N = J+1$, $\mathbf{x}_j = \mathbf{S}_j$ for $j \in\{1\ldots J\}$, $\mathbf{x}_{M+1} = \lambda$, $H$ as the data fidelity term, $f_j(\cdot) = \delta_{\mathcal{E}_{j}}(.)$ for $j \in\{1\ldots J\}$ and $f_{J+1}(\cdot) = \delta_{\mathcal{E}_{J+1}}(\cdot)=\delta_{\mathbb{R}}(\cdot) = 0$ (there is no constraint on $\lambda$). This match allows to apply PALM to compute multi-layer sparse approximations, with guaranteed convergence to a stationary point. 


\subsubsection{Projection operator} PALM relies on projections onto the constraint sets for each factor at each iteration, so the projection operator should be simple and easy to compute. For example, in the case where the $\mathcal{E}_j$s are sets of sparse normalized matrices, namely $\mathcal{E}_j = \{\mathbf{S} \in \mathbb{R}^{a_j \times a_{j+1}} : \left\Vert \text{vec}(\mathbf{S}) \right\Vert_0 \leq s_j, \left\Vert \mathbf{S} \right\Vert_F = 1\}$ for $j \in\{1\ldots J\}$, then the projection operator $P_{\mathcal{E}_j}(\cdot)$ simply keeps the $s_j$ greatest entries (in absolute value) of its argument, sets all the other entries to zero, and then normalizes its argument so that it has unit norm (the proof is given in Appendix \ref{app:projop}). Regarding $\mathcal{E}_{J+1} = \mathbb{R}$, the projection operator is the identity mapping. The projection operators for other forms of sparsity constraints that could be interesting in concrete applications are also given in Appendix \ref{app:projop}: Proposition~\ref{prop:proxsparsity} covers the following examples: 
\begin{itemize}
\item Global sparsity constraints.
\item Row or column sparsity constraints.
\item constrained support.
\item Triangular matrices constraints.
\item Diagonal matrices constraints.
\end{itemize}
Proposition~\ref{prop:piecewisesparsity} covers in addition:
\begin{itemize}
\item Circulant, Toeplitz or Hankel matrices with fixed support or prescribed sparsity.
\item Matrices that are constant by row or column.
\item More general classes of piece-wise constant matrices with possible sparsity constraints.
\end{itemize}

\subsubsection{Gradient and Lipschitz modulus}
 To specify the iterations of PALM specialized to the multi-layer sparse approximation problem, let us fix the iteration $i$ and the factor $j$, and denote $\mathbf{S}^i_j$  the factor being updated, $\mathbf{L} := \prod_{\ell=j+1}^{J} \mathbf{S}^{i}_\ell$ what is on its left and $\mathbf{R} := \prod_{\ell=1}^{j-1} \mathbf{S}^{i+1}_\ell$ what is on its right (with the convention $\prod_{\ell \in \varnothing} \mathbf{S}_\ell = \mathbf{Id}$). These notations give, when updating the $j$th factor $\mathbf{S}^i_{j}$:
\(
H(\mathbf{S}_1^{i+1},\ldots,\mathbf{S}_{j-1}^{i+1},\mathbf{S}_j^i,\ldots,\mathbf{S}_{J}^{i},\lambda^i) = H(\mathbf{L},\mathbf{S}_j^i,\mathbf{R},\lambda^i) = \tfrac{1}{2}
\| \mathbf{A} - \lambda^i\mathbf{L}\mathbf{S}^i_j\mathbf{R} 
\|_F^2.
\)
 The gradient of this smooth part of the objective with respect to the $j$th factor reads:
\begin{equation*}
\nabla_{\mathbf{S}^i_j}H(\mathbf{L},\mathbf{S}_j^i,\mathbf{R},\lambda^i) = {\lambda^i}\mathbf{L}^T(\lambda^i\mathbf{L}\mathbf{S}^i_j\mathbf{R} - \mathbf{A})\mathbf{R}^T ,
\end{equation*}
which Lipschitz modulus with respect to $\left\Vert \mathbf{S}_j^i \right\Vert_F$ is $L_j(\mathbf{L},\mathbf{R},\lambda^i) = (\lambda^i)^2\left\Vert \mathbf{R} \right\Vert_2^2. \left\Vert \mathbf{L} \right\Vert_2^2$ (as shown in Appendix \ref{app:lipmod}). Once all the $J$ factors are updated, let us now turn to the update of $\lambda$. Denoting $\hat{\mathbf{A}} = \prod_{j=1}^{J} \mathbf{S}^{i+1}_j$ brings:
\mbox{
\(
H(\mathbf{S}_1^{i+1},\ldots,\mathbf{S}_{J}^{i+1},\lambda^i) = \tfrac{1}{2}\| \mathbf{A} - \lambda^i\hat{\mathbf{A}} \|_F^2,
\)
}
and the gradient with respect to $\lambda^i$ reads:
\begin{equation*}
\nabla_{\lambda^i}H(\mathbf{S}_1^{i+1},\ldots,\mathbf{S}_{J}^{i+1},\lambda^i) = \lambda^i \text{Tr}(\hat{\mathbf{A}}^T\hat{\mathbf{A}}) - \text{Tr}(\mathbf{A}^T\hat{\mathbf{A}}). 
\end{equation*}

 
 \subsubsection{Default initialization, and choice of the step size}\label{sec:defaultinit} Except when specified otherwise, the default initialization is with $\lambda^{0}=1$, $\mathbf{S}_1^0 = \mathbf{0}$, and $\mathbf{S}_j^0 = \mathbf{Id}$ for $j \geq 2$, with the convention that for rectangular matrices the identity has ones on the main diagonal and zeroes elsewhere. In practice the step size is chosen by taking $c_j^i = (1+\alpha) .(\lambda^i)^2\left\Vert \mathbf{R} \right\Vert_2^2. \left\Vert \mathbf{L} \right\Vert_2^2$ with $\alpha = 10^{-3}$. Such a determination of the step size is computationally costly, and alternatives could be considered in applications (a decreasing step size rule for example).
 
 \subsubsection{Summary}
 An explicit version of the algorithm, called PALM for Multi-layer Sparse Approximation  (\texttt{palm4MSA}), is given in Figure~\ref{fig:algo_explicit}, in which the factors are updated alternatively by a projected gradient step (line 6) with a step-size controlled by the Lipschitz modulus of the gradient (line 5). We can solve for $\lambda$ directly at each iteration (line 9) because of the absence of constraint on it (thanks to the second part of the convergence condition (v) of PALM).

\begin{figure}[htbp]
PALM for Multi-layer Sparse Approximation (\texttt{palm4MSA})
\begin{boxedalgorithmic}[1] 
\REQUIRE{Operator $\mathbf{A}$; desired number of factors $J$; constraint sets $\mathcal{E}_j, \: j \in \{1\ldots J\}$; initialization $\{\mathbf{S}_j^0\}_{j=1}^J$, $\lambda^{0}$; stopping criterion (e.g., number of iterations $N$).}
\FOR{$i=0$ to $N-1$}
\FOR{$j=1$ to $J$}
\STATE  $\mathbf{L} \leftarrow \prod_{\ell=j+1}^{J} \mathbf{S}^{i}_\ell$
\STATE  $\mathbf{R} \leftarrow \prod_{\ell=1}^{j-1} \mathbf{S}^{i+1}_\ell$
\STATE Set $c^i_j > (\lambda^i)^2\left\Vert \mathbf{R} \right\Vert_2^2. \left\Vert \mathbf{L} \right\Vert_2^2$
\STATE $\mathbf{S}^{i+1}_j \leftarrow P_{\mathcal{E}_j}\Big(\mathbf{S}^{i}_j - \frac{1}{c^i_j}\lambda^{i}\mathbf{L}^T(\lambda\mathbf{L}\mathbf{S}^i_j\mathbf{R} - \mathbf{A})\mathbf{R}^T\Big)$
\ENDFOR
\STATE  $\hat{\mathbf{A}} \leftarrow \prod_{j=1}^{J} \mathbf{S}^{i+1}_j$
\STATE  $\lambda^{i+1} \leftarrow \frac{\text{Tr}(\mathbf{A}^T\hat{\mathbf{A}})}{\text{Tr}(\hat{\mathbf{A}}^T\hat{\mathbf{A}})}$
\ENDFOR
\ENSURE The estimated factorization:

 $\lambda^{N}$,$\{\mathbf{S}^{N}_j\}_{j=1}^{J}$ = \texttt{palm4MSA}($\mathbf{A}$, $J$, $\{\mathcal{E}_j\}_{j=1}^{J}$, \ldots)
\end{boxedalgorithmic}
\caption{PALM algorithm for multi-layer sparse approximation.}
\vspace{-0.5cm}
\label{fig:algo_explicit}
\end{figure}

\section{Hierarchical factorization}
\label{ssec:practical}
The algorithm presented in Figure \ref{fig:algo_explicit} factorizes an input matrix corresponding to an operator of interest into $J$ sparse factors and converges to a stationary point of the problem stated in \eqref{eq:actualproblemscal}. In practice, one is only interested in the stationary points where the data fitting term of the cost function is small, however as for any generic non-convex optimization algorithm there is no general convergence guarantee to such a stationary point.  This fact is illustrated by a very simple experiment where the algorithm \texttt{palm4MSA} is applied to an input operator $\mathbf{A} \in \mathbb{R}^{n \times n}$ with a known sparse form $\mathbf{A}  = \prod_{j=1}^N \mathbf{S}_j$, such as the Hadamard transform (in that case $N = \log_2n$).
The naive approach consists in setting directly $J=N$ in \texttt{palm4MSA}, and setting the constraints so as to reflect the actual sparsity of the true factors (as depicted in Figure~\ref{fig:facthadamard}). This simple strategy performs quite poorly in practice for most initializations, and the attained local minimum is very often not satisfactory (the data fidelity part of the objective function is large). 

\subsection{Parallel with deep learning}
Similar issues were faced in the neural network community, where it was found difficult to optimize the weights of neural networks comprising many hidden layers (called deep neural networks, see \cite{Bengio2009} for a survey on the topic). Until recently, deep networks were often neglected in favor of shallower architectures. However in the last decade, it was proposed \cite{Hinton2006} to optimize the network not as one big block, but one layer at a time, and then globally optimizing the whole network using gradient descent. This heuristic was shown experimentally to work well on various tasks \cite{Bengio2007}. More precisely, what was proposed is to perform first a \emph{pre-training} of the layers (each being fed the features produced by the one just below, and the lowermost being fed the data), in order to initialize the weights in a good region to perform then a global \emph{fine tuning} of all layers by simple gradient descent.




\subsection{Proposed hierarchical algorithm}

We noticed experimentally that taking fewer factors ($J$ small) and allowing more non-zero entries per factor led to better approximations. This suggested to adopt a hierarchical strategy reminiscent of pre-training of deep networks, in order to iteratively compute only factorization with $2$ factors. Indeed, when $\mathbf{A}  = \prod_{j=1}^{N} \mathbf{S}_j$ is the product of $N$ sparse factors, it is also the product $\mathbf{A}  = \mathbf{T}_1\mathbf{S}_1$ of $2$ factors with $\mathbf{T}_1 = \prod_{j=2}^{N} \mathbf{S}_j$, so that $\mathbf{S}_1$ is sparser than $\mathbf{T}_1$. 

\subsubsection{Optimization strategy} 
The proposed hierarchical strategy consists in iteratively factorizing the input matrix $\mathbf{A}$ into $2$ factors, one being sparse (corresponding to $\mathbf{S}_1$), and the other less sparse (corresponding to $\mathbf{T}_1$). The process is repeated on the less sparse factor $\mathbf{T}_1$ until the desired number $J$ of factors is attained. At each step, a global optimization of all the factors introduced so far can be performed in order to fit the product to the original operator $\mathbf{A}$. 

\subsubsection{Choice of sparsity constraint} 
A natural question is that of how to tune the sparsity of the factors and residuals along the process. Denoting $\mathbf{T}_\ell = \prod_{j=\ell+1}^{J} \mathbf{S}_j$, a simple calculation shows that if we expect each $\mathbf{S}_j$ to have roughly $\mathcal{O}(k)$ non-zero entries per row, then $\mathbf{T}_\ell$ cannot have more than $\mathcal{O}(k^{J-(\ell+1)})$ non-zero entries per row. This suggests to decrease exponentially the number of non-zero entries in $\mathbf{T}_\ell$ with $\ell$ and to keep constant $\mathcal{O}(k)$ the number of non-zero entries per row in $\mathbf{S}_j$. This choice of the sparsity constraints is further studied with experiments in Section~\ref{sec:MEG}.

\subsubsection{Implementation details} 
\label{sssec:impdetails}
The proposed hierarchical strategy\footnote{A toolbox implementing all the algorithms and experiments performed in this paper is available at \href{http://faust.gforge.inria.fr}{\texttt{http://faust.gforge.inria.fr}}\newline All experiments were performed in Matlab on an laptop with an intel(R) core(TM) i7-3667U @ 2.00GHz (two cores).} is summarized in the algorithm given in Figure~\ref{fig:algo_hierarchical}, where the constraint sets related to the two factors need to be specified for each step: $\tilde{\mathcal{E}}_\ell$ denotes the constraint set related to the left factor $\mathbf{T}_\ell$, and $\mathcal{E}_\ell$ the one for the right factor $\mathbf{S}_\ell$ at the $\ell$th factorization. 
\begin{figure}[htbp]
Hierarchical factorization ~\vspace{0.6mm} \\ 
\begin{boxedalgorithmic}[1]
\REQUIRE{Operator $\mathbf{A}$; desired number of factors $J$; constraint sets $\tilde{\mathcal{E}}_\ell$ and $\mathcal{E}_\ell, \: \ell \in \{1\ldots J-1\}$.}\STATE $\mathbf{T}_0 \leftarrow \mathbf{A}$
\FOR{$\ell=1$ to $J-1$} 
\STATE Factorize the residual $\mathbf{T}_{\ell-1}$ into $2$ factors:

 $\lambda'$,$\{\mathbf{F}_2,\mathbf{F}_1\}$ = \texttt{palm4MSA}($\mathbf{T}_{\ell-1}$, $2$, $\{\tilde{\mathcal{E}}_\ell,\mathcal{E}_\ell\}$, \texttt{init=default})
\STATE $\mathbf{T}_\ell \leftarrow \lambda'\mathbf{F}_2$ and  $\mathbf{S}_\ell \leftarrow \mathbf{F}_1$
\STATE Global optimization:

 $\lambda$,$\big\{\mathbf{T}_\ell,\{\mathbf{S}_j\}_{j=1}^{\ell}\big\}$ = \texttt{palm4MSA}($\mathbf{A}$, $\ell+1$, $\big\{\tilde{\mathcal{E}}_\ell,\{\mathcal{E}_j\}_{j=1}^{\ell}\big\}$,\texttt{init=current}) 
\ENDFOR
\STATE  $\mathbf{S}_J \leftarrow \mathbf{T}_{J-1}$
\ENSURE The estimated factorization:  $\lambda$,$\{\mathbf{S}_j\}_{j=1}^{J}$.
\end{boxedalgorithmic}
\caption{Hierarchical factorization algorithm.}
\vspace{-0.5cm}
\label{fig:algo_hierarchical}
\end{figure}
Roughly we can say that line $3$ of the algorithm is here to yield complexity savings. Line $5$ is here to improve data fidelity: this global optimization step with \texttt{palm4MSA} is initialized with the current values of  $\mathbf{T}_\ell$ and $\{\mathbf{S}_j\}_{j=1}^{\ell}$. \modif{The hierarchical strategy uses \texttt{palm4MSA} $J-1$ times with an increasing number of factors, and with a good initialization provided by the factorization in two factors. This makes its cost roughly $J-1$ times greater than the cost of the basic \texttt{palm4MSA} with $J$ factors.} 

In greedy layerwise training of deep neural networks,  the factorizations in two (line~3) would correspond to the pre-training and the global optimization (line~5) to the fine tuning. 

 \begin{remark}
The hierarchical strategy can also be applied the other way around (starting {\em from the left}), just by transposing the input. We only present here the version that starts {\em from the right} because the induced notations are simpler. It is also worth being noted that stopping criteria other than the total number of factors can be set. For example, we could imagine to keep factorizing the residual until the approximation error at the global optimization step starts rising or exceeds some pre-defined threshold. 
\end{remark}

\subsection{Illustration: reverse-engineering the Hadamard transform}
\label{ssec:hadamard}
As a first illustration of the proposed approach, we tested the hierarchical factorization algorithm of Figure~\ref{fig:algo_hierarchical} when $\mathbf{A}$ is the dense square matrix associated to the Hadamard transform in dimension $n=2^N$. The algorithm was run with $J = N$ factors, $\tilde{\mathcal{E}}_\ell = \{\mathbf{T} \in \mathbb{R}^{n \times n},  \left\Vert \mathbf{T} \right\Vert_0 \leq \frac{n^2}{2^{\ell}}, \left\Vert \mathbf{T} \right\Vert_F = 1 \}$, and $\mathcal{E}_\ell = \{\mathbf{S} \in \mathbb{R}^{n \times n},  \left\Vert \mathbf{S} \right\Vert_0 \leq 2n, \left\Vert \mathbf{S} \right\Vert_F = 1 \}$. 

In stark contrast with the direct application of  \texttt{palm4MSA} with $J=N$, an exact factorization is achieved. Indeed, the first step reached an exact factorization $\mathbf{A} = \mathbf{T}_{1} \mathbf{S}_{1}$ independently of the initialization. With the {\em default} initialization (Section~\ref{sec:defaultinit}), the residual $\mathbf{T}_{1}$ was observed to be still exactly factorizable. All steps ($\ell>1$) indeed also yielded exact factorizations $\mathbf{T}_{\ell-1} = \mathbf{T}_{\ell}\mathbf{S}_{\ell}$, provided the default initialization  was used at each step $\ell \in \{1,\dots,J-1\}$.

 Figure~\ref{fig:hierarchical_strategy} illustrates the result of the proposed hierarchical strategy  
in dimension $n=32$. The obtained factorization is exact and \emph{as good as the reference one} (cf Figure~\ref{fig:facthadamard}) in terms of complexity savings.  The running time of the factorization algorithm is less than a second. Factorization of the Hadamard matrix in dimension up to $n=1024$ showed identical behaviour, with running times $\mathcal{O}(n^{2})$ up to ten minutes.

\begin{figure}[htbp]
\centering
\includegraphics[width=\columnwidth]{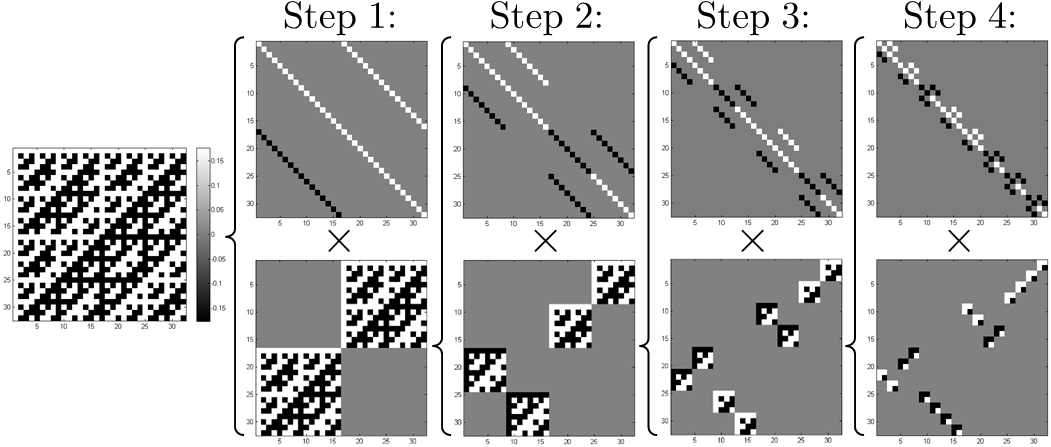}
\caption{Hierarchical factorization of the Hadamard matrix of size $32\times 32$. The matrix is iteratively factorized in 2 factors, until we have $J=5$ factors, each having $s=64$ non-zero entries.}
\vspace{-0.5cm}
\label{fig:hierarchical_strategy}
\end{figure}

\section{Accelerating inverse problems}
\label{sec:MEG}
A natural application of FA$\mu$STs is linear inverse problems, where a high-dimensional vector
$\boldsymbol{\gamma}$ needs to be retrieved from some observed data $\mathbf{y} \approx \mathbf{M} \boldsymbol{\gamma}$. As already evoked in Section~\ref{ssec:advantages}, iterative proximal algorithms can be expected to be significantly sped up if $\mathbf{M}$ is well approximated with a FA$\mu$ST of low relative complexity, for example using the proposed hierarchical factorization algorithm applied to $\mathbf{A} = \mathbf{M}$. 

In practice, one needs to specify the total number of factors $J$ and the constraint sets $\tilde{\mathcal{E}}_\ell$, $\mathcal{E}_\ell$.
A preliminary study on synthetic data was carried out in our technical report \cite{Lemagoarou2014}, showing that a flexible trade-off between relative complexity and adaptation to the input matrix can be achieved. Here we leverage the rule of thumb presented in Section~\ref{ssec:algdetails} to deepen the investigation of this question for a matrix $\mathbf{M}$ arising in a real-world biomedical linear inverse problem.


\subsection{Factorization compromise: MEG operator}
\label{ssec:mf}

In this experiment, we explore the use of FA$\mu$ST
in the context of functional brain imaging using magnetoencephalography (MEG)
and electroencephalography (EEG) signals.
Source imaging with MEG and EEG
delivers insights into the active brain at a millisecond time scale in a non-invasive way.
To achieve this, one needs to solve the bioelectromagnetic inverse problem. It is a high
dimensional ill-posed regression problem requiring proper regularization.
As it is natural to assume that a limited set of brain foci are active
during a cognitive task, sparse focal source configurations are commonly promoted
using convex sparse priors~\cite{Haufe2008,Gramfort2012}.
The bottleneck in the optimization algorithms are
the dot products with the forward matrix and its transpose.

\modif{The objective of this experiment is to observe achievable trade-offs between relative complexity and accuracy. To this end, we consider an MEG gain matrix $\mathbf{M} \in \mathbb{R}^{204\times 8193}$ ($m=204$ and $n=8193$), computed using the MNE software \cite{Gramfort2013} implementing a Boundary Element Method (BEM). In this setting, sensors and sources are not on a regular spatial grid. Note that in this configuration, one cannot easily rely on classical operator compression methods presented in sections~\ref{sssec:loclow}, that rely on the analytic expression of the kernel underlying $\mathbf{M}$, or on those presented in section~\ref{sssec:wavcomp}, that rely on some regularity of the input and output domains to define wavelets. Hence, in order to observe the complexity/accuracy trade-offs, $\mathbf{M}$ was factorized into $J$ sparse factors using the hierarchical factorization algorithm of Figure~\ref{fig:algo_hierarchical}.}

\subsubsection{Settings}
The rightmost factor $\mathbf{S}_1$ was of the size of $\mathbf{M}$, but with $k$-sparse columns, corresponding to the constraint set $\mathcal{E}_1 = \{\mathbf{S} \in \mathbb{R}^{204 \times 8193}, \left\Vert \mathbf{s}_i \right\Vert_0 \leq k,  \left\Vert \mathbf{S} \right\Vert_F = 1 \}$. All other factors $\mathbf{S}_j$, $j \in \{2,\ldots,J\}$ were set square, with global sparsity $s$, i.e. $\mathcal{E}_\ell = \{\mathbf{S} \in \mathbb{R}^{204 \times 204}, \left\Vert \mathbf{S} \right\Vert_0 \leq s, \left\Vert \mathbf{S} \right\Vert_F = 1 \}$. 

The ``residual'' at each step $\mathbf{T}_\ell$, $\ell \in \{1,\ldots,J-1\}$ was also set square, with global sparsity geometrically decreasing  with $\ell$, controlled by two parameters $\rho$ and $P$. This corresponds to the constraint sets\footnote{Compared to a preliminary version of this experiment  \cite{Lemagoarou2015a} where the residual was normalized columnwise at the first step, here it is normalized globally. This leads to slightly better results.} $\tilde{\mathcal{E}}_\ell = \{\mathbf{T} \in \mathbb{R}^{204 \times 204}, \left\Vert \mathbf{T} \right\Vert_0 \leq P\rho^{\ell-1}, \left\Vert \mathbf{T} \right\Vert_F = 1 \}$ for $\ell \in \{1,\ldots,J-1\}$. 
The controlling parameters are set to:
\begin{itemize}
\item Number of factors: $J \in \{2,\dots,10\}$.
\item Sparsity of the rightmost factor: $k\in\{5,10,15,20,25,30\}$.
\item \modif{Sparsity of the other factors: $ s \in \{2m,4m,8m\}$.}
\item Rate of decrease of the residual sparsity: $\rho = 0.8$. 
\end{itemize}
The parameter $P$ controlling the global sparsity in the residual was found to have only limited influence, and was set to $P = 1.4 \times m^{2}$. 
Other values for $\rho$ were tested, leading to slightly different but qualitatively similar complexity/accuracy trade-offs not shown here. The factorization setting is summarized in Figure~\ref{fig:setting}, where the sparsity of each factor is explicitly given. 

\begin{figure}[tbp]
    \centering
        \includegraphics[width=1\columnwidth]{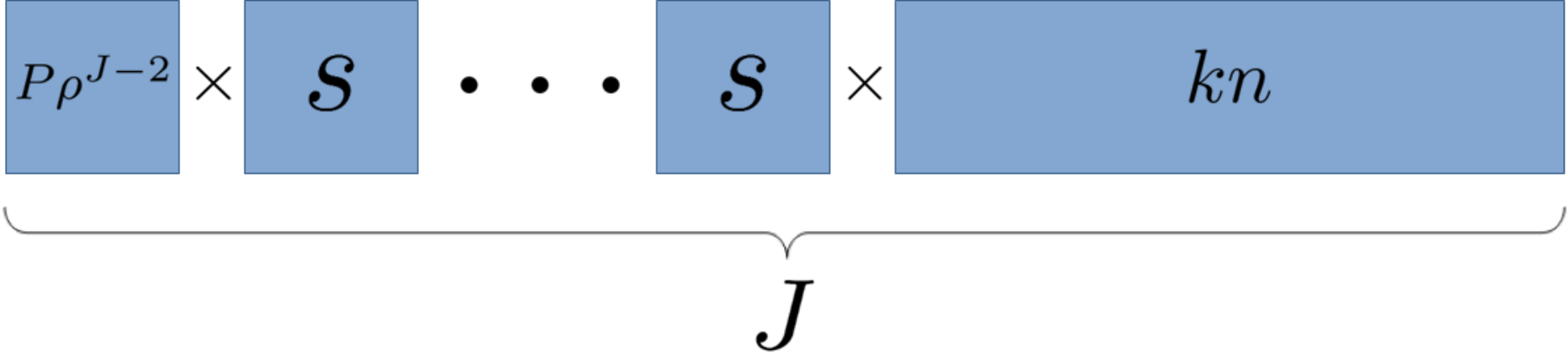}
    \caption{Factorization setting: each factor is represented with its total sparsity.}
    \vspace{-0.5cm}  
    \label{fig:setting}
\end{figure}

\begin{figure*}[htb]
    \centering
        \includegraphics[width=1\textwidth]{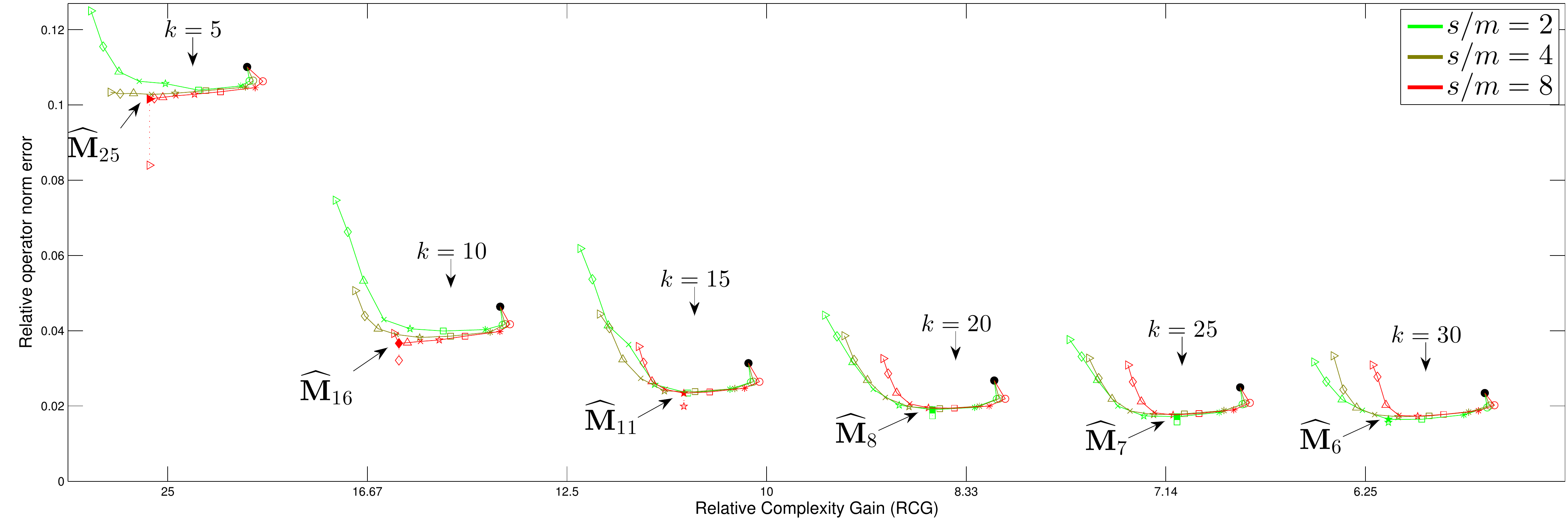}
        \vspace{-0.7cm}
    \caption{Results of the factorization of an $m \times n = 204\times 8193$ MEG matrix. The shape of the symbols denotes the number of factors $J$ ($\bullet: J=2$, $\fullmoon: J=3$, $\ast: J=4$, $\square: J=5$, \protect\raisebox{-0.5mm}{\FiveStarOpen} $: J=6$, $\times: J=7$, $\vartriangle: J=8$, $\diamond: J=9$, $\vartriangleright: J=10$), and the color the value of the parameter $s$. 
    } 
    \vspace{-0.5cm} 
    \label{fig:MEG}
\end{figure*}

\subsubsection{Results} 

Factorizations were computed for $127$ parameter settings. The computation time for each factorization was around $(J-1)\times 10$ minutes. Figure~\ref{fig:MEG} displays the trade-off between speed (the RCG measure~\eqref{RC}) and approximation error:
\begin{equation}
\text{RE} := \frac{\big\Vert\mathbf{M} - \lambda\prod_{j=1}^J\mathbf{S}_j\big\Vert_2}{\left\Vert\mathbf{M}\right\Vert_2} ,
\label{RE}
\end{equation} 
of each obtained FA$\mu$ST. We observe that: 
\begin{itemize}
\item The overall relative complexity of the obtained factorization is essentially controlled by the parameter $k$. This seems natural, since $k$ controls the sparsity of the rightmost factor which is way larger than the other ones. 
\item The  trade-off between complexity and approximation for a given $k$ is mainly driven by the number of factors $J$: higher values of $J$ lead to lower relative complexities, but a too large $J$ leads to a higher relative error.  Taking $J=2$ (black dots) never yields the best compromise, hence the relevance of truly \emph{multi-layer} sparse approximations. 
\item For a fixed $k$, one can distinguish nearby trade-off curves corresponding to different sparsity levels $s$ of the intermediate factors.  The parameter $s$ actually controls the horizontal spacing between two consecutive points on the same curve: a higher $s$ allows to take a higher $J$ without increasing the error, but in turn leads to a higher relative complexity for a given number $J$ of factors.   
 \end{itemize}
In summary, one can distinguish as expected a trade-off between relative complexity and approximation error. The configuration exhibiting the lowest relative error for each value of $k$ is highlighted on Figure~\ref{fig:MEG}, this gives $\widehat{\mathbf{M}}_{25}$, $\widehat{\mathbf{M}}_{16}$, $\widehat{\mathbf{M}}_{11}$, $\widehat{\mathbf{M}}_{8}$, $\widehat{\mathbf{M}}_{7}$, $\widehat{\mathbf{M}}_{6}$, where the subscript indicates the achieved RCG (rounded to the closest integer). For example, $\widehat{\mathbf{M}}_{6}$ can multiply vectors with $6$ times less flops than $\mathbf{M}$ (saving $84\%$ of computation), and $\widehat{\mathbf{M}}_{25}$ can multiply vectors with $25$ times less flops than $\mathbf{M}$ ($96\%$ savings). These six matrices are those appearing on Figure~\ref{fig:compSVD} to compare FA$\mu$STs to the truncated SVD. They will next be used to solve an inverse problem and compared to results obtained with $\mathbf{M}$. 

\begin{remark} Slightly smaller approximation errors can be obtained by imposing a global sparsity constraint to the rightmost factor, i.e., $\mathcal{E}_1 = \{\mathbf{S} \in \mathbb{R}^{204 \times 8193}, \left\Vert \mathbf{S} \right\Vert_0 \leq kn,  \left\Vert \mathbf{S} \right\Vert_F = 1 \}$. This is shown on Figure~\ref{fig:MEG} by the points linked by a dashed line to the six matrices  outlined above. However, such a global sparsity constraint also allows the appearance of null columns in the approximations of $\mathbf{M}$, which is undesirable for the application considered next.
\end{remark}

\subsection{Source localization experiment}
\label{ssec:sl}
We now assess the impact of replacing the MEG gain matrix $\mathbf{M} \in \mathbb{R}^{204 \times 8193}$ by a FA$\mu$ST approximation for brain source localization. For this synthetic experiment, two brain sources chosen located uniformly at random were activated with gaussian random weights, giving a $2$-sparse vector $\boldsymbol{\gamma} \in \mathbb{R}^{8193}$, whose support encodes the localization of the sources. Observing $\mathbf{y} := \mathbf{M}\boldsymbol{\gamma}$, the objective is to estimate (the support of) $\boldsymbol{\gamma}$. The experiment then amounts to solving the inverse problem to get $\hat{\boldsymbol{\gamma}}$ from the measurements $\mathbf{y} = \mathbf{M}\boldsymbol{\gamma} \in \mathbb{R}^{204}$, using either $\mathbf{M}$ or a FA$\mu$ST $\widehat{\mathbf{M}}$ during the recovery process. 

Three recovery methods were tested:  Orthogonal Matching Pursuit (OMP) \cite{Tropp2007} (choosing $2$ atoms), $\ell_1$-regularized least squares ($\text{l}_1\text{ls}$) \cite{Kim2007} and Iterative Hard Thresholding (IHT) \cite{Blumensath2008}).  They yielded qualitatively similar results, and for the sake of conciseness we present here only the results for OMP.


The matrices used for the recovery are
the actual matrix $\mathbf{M}$ and
its FA$\mu$ST approximations $\widehat{\mathbf{M}}_{25}$, $\widehat{\mathbf{M}}_{16}$, $\widehat{\mathbf{M}}_{11}$, $\widehat{\mathbf{M}}_{8}$, $\widehat{\mathbf{M}}_{7}$ and $\widehat{\mathbf{M}}_{6}$. 
The expected computational gain of using a FA$\mu$ST instead of $\mathbf{M}$ is of the order of RCG, since the computational cost of OMP is dominated by products with $\mathbf{M}^T$. 

\begin{figure*}[htbp]
    \centering
        \includegraphics[width=\textwidth]{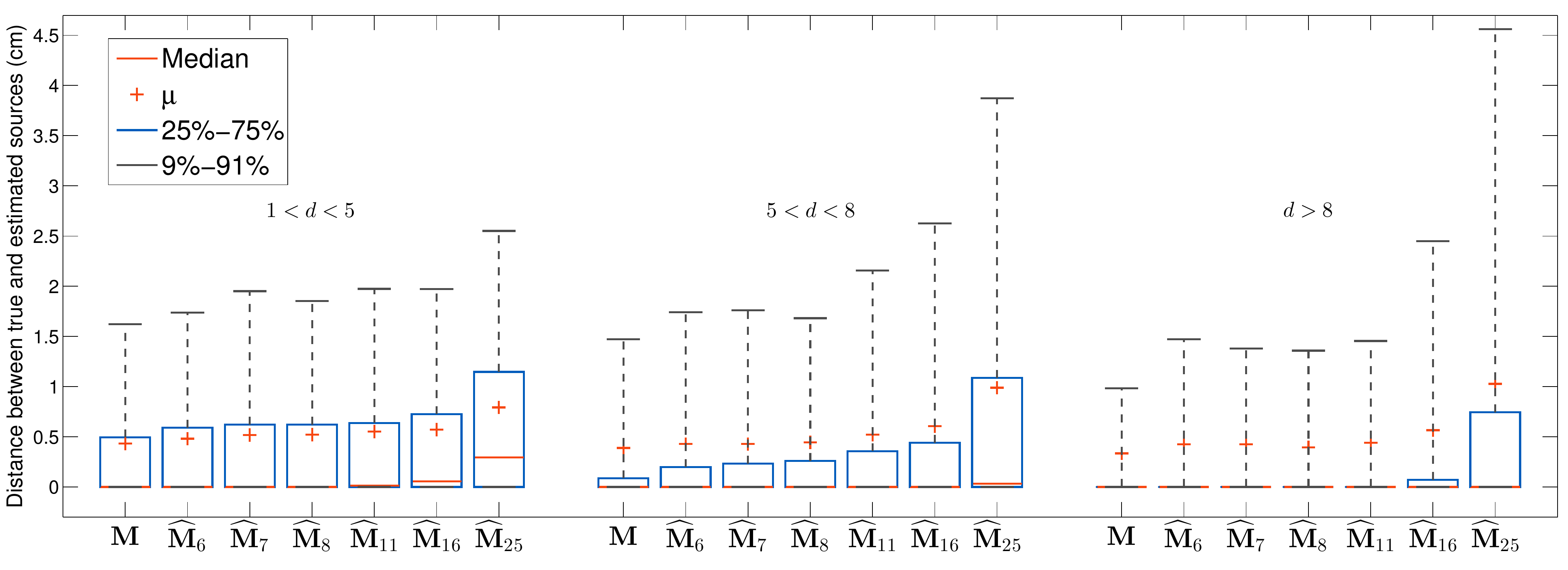}
        \vspace{-0.7cm}
\caption{Localization performance (distance between actual and retrieved source) obtained with various matrices, for different distances between actual sources.}
\vspace{-0.5cm}
    \label{fig:BSL}
\end{figure*}

Three configurations were considered when generating the location of the sources, in term of distance $d$ (in centimeters) between the sources. For each configuration,  $500$ vectors $\mathbf{y} = \mathbf{M} \boldsymbol{\gamma}$ were generated, and OMP was run using each matrix. 
The distance between each actual source and the closest retrieved source was measured. Figure~\ref{fig:BSL} displays the statistics of this distance for all scenarios:
\begin{itemize}
\item As expected, localization is better when the sources are more separated, independently of the choice of matrix.
\item Most importantly, the performance is almost as good when using FA$\mu$STs $\widehat{\mathbf{M}}_{6}$, $\widehat{\mathbf{M}}_{7}$, $\widehat{\mathbf{M}}_{8}$ and $\widehat{\mathbf{M}}_{11}$ than when using the actual matrix $\mathbf{M}$, although the FA$\mu$STs are way more computationally efficient ($6$ to $11$ times less computations). For example, in the case of well separated sources ($d>8$), the FA$\mu$STs allow to retrieve exactly the sought sources more than $75\%$ of the time, which is almost as good as when using the actual matrix $\mathbf{M}$. 
\item The performance with the two other FA$\mu$STs $\widehat{\mathbf{M}}_{16}$ and $\widehat{\mathbf{M}}_{25}$ is a bit poorer, but they are even more computationally efficient matrix ($16$ and $25$ times less computations).  For example, in the case of well separated sources ($d>8$), they allow to retrieve exactly the sought sources more than $50\%$ of the time. 
\end{itemize}
These observations confirm it is possible to slighlty trade-off localization performance for substantial computational gains, and that FA$\mu$STs can be used to speed up inverse problems without a large precision loss.

\section{Learning fast dictionaries}
\label{sec:diclearn}

Multi-layer sparse approximations of operators are particularly suited for choosing efficient dictionaries for data processing tasks. 
\subsection{Analytic vs. learned dictionaries}

Classically, there are two paths to choose a dictionary for sparse signal representations  \cite{Rubinstein2010}. 

Historically, the only way to come up with a dictionary was to analyze mathematically the data and derive a ``simple"  formula to construct the dictionary. 
Dictionaries designed this way are called \emph{analytic dictionaries}  \cite{Rubinstein2010} (e.g., associated to Fourier, wavelets and Hadamard transforms). Due to the relative simplicity of analytic dictionaries, they usually have a known sparse form such as the Fast Fourier Transform (FFT) \cite{CooleyTukey1965} or the Discrete Wavelet Transform (DWT) \cite{Mallat1989}.

On the other hand, the development of modern computers allowed the surfacing of automatic methods that learn a dictionary directly from the data \cite{Engan1999,Aharon2006,Mairal2010}. 
Given some raw data $\mathbf{Y} \in \mathbb{R}^{m \times L}$, the principle of dictionary learning is to approximate $\mathbf{Y}$ by the product of a dictionary $\mathbf{D} \in \mathbb{R}^{m \times n}$ and a coefficient matrix $\boldsymbol{\Gamma} \in \mathbb{R}^{n \times L}$ with sparse columns:
\begin{equation*}
\mathbf{Y} \approx \mathbf{D}\boldsymbol{\Gamma}.
\end{equation*}
Such \emph{learned dictionaries} are usually well adapted to the data at hand. However, being in general dense matrices with no apparent structure, they do not lead to fast algorithms and are costly to store. 
 We typically have $L \gg \max(m,n)$ (for sample complexity reasons), which implies to be very careful about the computational efficiency of learning in that case.
\subsection{The best of both worlds}
\label{ssec:diclearnalgo}
Can one design dictionaries as well adapted to the data as learned dictionaries, while as fast to manipulate and as cheap to store as analytic ones? This question has begun to be explored recently \cite{Rubinstein2010a,Chabiron2013}, and actually amounts to learning of dictionary that are FA$\mu$STs. More precisely, given $\mathbf{Y}$, the objective is to learn a dictionary being a FA$\mu$ST (as in \eqref{eq:spop}):
\begin{equation*}
\mathbf{D} = \prod_{j=1}^J \mathbf{S}_j.
\end{equation*}

This can be done by inserting a dictionary factorization step into the traditional structure of dictionary learning algorithms \cite{Rubinstein2010}, as illustrated on Figure~\ref{fig:DLstruct}.
\begin{figure}[b]
\centering
\includegraphics[width=\columnwidth]{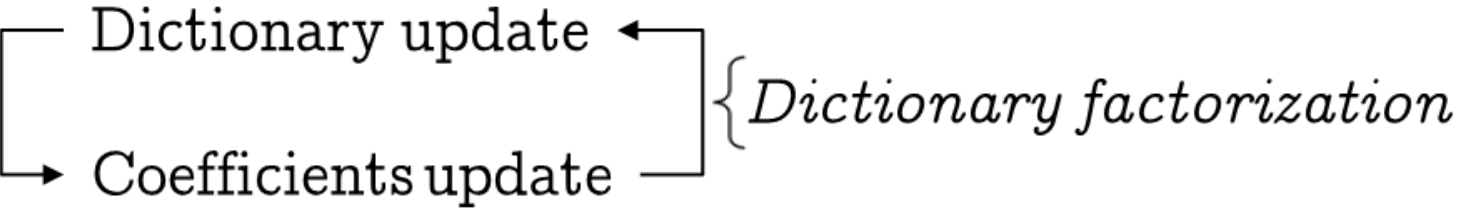}
\caption{Classical dictionary learning algorithm structure (in roman). In italic, the added dictionary factorization step, specific to the approach presented here.}
\vspace{-0.1cm}
\label{fig:DLstruct}
\end{figure}
A consequence is that the coefficients update can be sped up by exploiting the  FA$\mu$ST structure of the dictionary. The approach described below uses a batch method for dictionary update, but the approach is a priori also compatible with stochastic gradient descent in the dictionary update for even more efficiency.

In practice we propose to slightly modify the hierarchical factorization algorithm of Figure~\ref{fig:algo_hierarchical}. The idea is to take a dictionary $\mathbf{D}$ learned on some training data $\mathbf{Y}$ (with any classical dictionary learning method, such as K-SVD \cite{Aharon2006}) and to hierarchically factorize it, taking into account and jointly updating the coefficients matrix $\boldsymbol{\Gamma}$. 

The resulting hierarchical factorization algorithm adapted to dictionary learning is given in Figure~\ref{fig:algo_hierarchical_DL}. The only differences with the hierarchical factorization algorithm given previously is that the coefficients matrix is taken into account (but kept fixed) in the global optimization step, and that an update of the coefficients by sparse coding is added after this global optimization step\modif{, in order to keep the error with respect to the data matrix low}. This sparse coding step can actually be done by any algorithm (OMP, IHT, ISTA...), denoted by the general \texttt{sparseCoding} algorithm in Figure~\ref{fig:algo_hierarchical_DL}. 


\begin{remark}
As noted in section~\ref{sssec:impdetails}, the dictionary factorization is presented here starting from the right. It could as well be performed starting from the left.
\end{remark}

\begin{figure}[htbp]
Hierarchical factorization for dictionary learning  \\ 
\begin{boxedalgorithmic}[1]
\REQUIRE{Data matrix $\mathbf{Y}$; Initial dictionary $\mathbf{D}$ and coefficients $\boldsymbol{\Gamma}$ (e.g., from K-SVD); desired number of factors $J$; constraint sets $\tilde{\mathcal{E}}_\ell$ and $\mathcal{E}_\ell$, $\ell \in \{1\ldots J-1\}$.}\STATE $\mathbf{T}_0 \leftarrow \mathbf{D}$, $\boldsymbol{\Gamma}_0 \leftarrow \boldsymbol{\Gamma}$
\FOR{$\ell=1$ to $J-1$} 
\STATE {\em Dictionary factorization}: 
factorize the residual $\mathbf{T}_{\ell-1}$ into $2$ factors

 $\lambda'$,$\{\mathbf{F}_2,\mathbf{F}_1\}$ = \texttt{palm4MSA}($\mathbf{T}_{\ell-1}$, $2$, $\{\tilde{\mathcal{E}}_\ell,\mathcal{E}_\ell\}$, \texttt{init=default})\\
$\mathbf{T}_\ell \leftarrow \lambda'\mathbf{F}_2$ and  $\mathbf{S}_\ell \leftarrow \mathbf{F}_1$
\STATE {\em Dictionary update}: global optimization

 $\lambda$,$\big\{\mathbf{T}_\ell,\{\mathbf{S}_j\}_{j=1}^{\ell},\boldsymbol{\Gamma}_{\ell-1}\big\}$ = \texttt{palm4MSA}($\mathbf{Y}$, $\ell+2$, $\big\{\tilde{\mathcal{E}}_\ell,\{\mathcal{E}_j\}_{j=1}^{\ell},\{\boldsymbol{\Gamma}_{\ell-1}\}\big\}$, \texttt{init=current}) 
 \STATE {\em Coefficients update}:

 $\boldsymbol{\Gamma}_{\ell}$ = \texttt{sparseCoding}($\mathbf{Y}$, $\mathbf{T}_{\ell}\prod_{j=1}^\ell\mathbf{S}_j$) 
\ENDFOR
\STATE  $\mathbf{S}_J \leftarrow \mathbf{T}_{J-1}$
\ENSURE The estimated factorization:  $\lambda$,$\{\mathbf{S}_j\}_{j=1}^{J}$.
\end{boxedalgorithmic}
\caption{Hierarchical factorization algorithm for dictionary learning.}
\vspace{-0.5cm}
\label{fig:algo_hierarchical_DL}
\end{figure}

\subsection{Image denoising experiment}
\label{sssec:imgdenois}
In order to illustrate the advantages of FA$\mu$ST dictionaries over classical dense ones, an image denoising experiment is performed here.
The experimental scenario for this task follows a simplified dictionary based image denoising workflow. First, $L = 10000$ 
patches $\mathbf{y}_{i}$ of size $8 \times 8$ (dimension $m=64$) are randomly picked from an input $512 \times 512$ noisy image (with  various noise levels, of variance $\sigma \in \{10,15,20,30,50\}$), and a dictionary is learned on these patches. Then the learned dictionary is used to denoise the entire input image by computing the sparse representation of all its patches in the dictionary using OMP, allowing each patch to use $5$ dictionary atoms. The image is reconstructed by averaging the overlapping patches. 

\noindent{\bf Experimental settings.} Several configurations were tested. The number of atoms $n$ 
was taken in $\{128,256,512\}$. Inspired by usual fast transforms, a number of factors $J$ close to the logarithm of the signal dimension $m=64$ 
was chosen, here $J = 4$. The sizes of the factors were: $\mathbf{S}_J,\ldots,\mathbf{S}_2 \in \mathbb{R}^{m \times m}$, 
$\mathbf{S}_1 \in \mathbb{R}^{m \times n}$, and $\boldsymbol{\Gamma} \in \mathbb{R}^{n \times L}$. The algorithm of Figure~\ref{fig:algo_hierarchical_DL} was used, with the initial dictionary learning being done by K-SVD \cite{Aharon2006} and \texttt{sparseCoding} being OMP, allowing each patch to use $5$ dictionary atoms. Regarding the constraint sets, we took them exactly like in section \ref{ssec:mf}, taking 
$s/m \in \{2,3,6,12\}$, $\rho \in \{0.4,0.5,0.7,0.9\}$, $P = 64^2$ and $k=s/m$. For each dictionary size $n$, this amounts to a total of sixteen different configurations leading to different relative complexity values. The stopping criterion for $\texttt{palm4MSA}$ was a number of iterations $N_{i}=50$. Note that the usage of OMP here and at the denoising stage is a bit abusive since the dictionary does not have unit-norm columns (the factors are normalized instead), but it was used anyway, resulting in a sort of weighted OMP, where some atoms have more weight than others.

\noindent{\bf Baselines.} 
The proposed method was compared to {\em Dense Dictionary Learning} (DDL). K-SVD is used here to perform DDL, but other algorithms have been tested (such as online dictionary learning \cite{Mairal2010}), leading to similar qualitative results. In order to assess the generalization performance and to be as close as possible to the matrix factorization framework studied theoretically in \cite{Gribonval2015}, DDL is performed following the same denoising workflow than our method (dictionary learned on $10000$ noisy patches used to denoise the whole image, allowing five atoms per patch). The implementation described in \cite{Rubinstein2008} was used, running $50$ iterations (empirically sufficient to ensure convergence). 

Note that better denoising performance can be obtained by inserting dictionary learning into a more sophisticated denoising workflows, see e.g. \cite{Elad2006}. State of the art denoising algorithms indeed often rely on clever averaging procedure called ``aggregation''. Our purpose here is primarily to illustrate the potential of the proposed FA$\mu$ST structure for denoising. While such workflows are fully compatible with the FA$\mu$ST structure, we leave the implementation and careful benchmarking of the resulting denoising systems to future work. 

\modif{ As a last baseline, we used the above denoising scheme with an overcomplete DCT of $128$, $256$ or $512$ atoms. }

\noindent{\bf Results.} 
The experiment is done on the standard image database taken from \cite{imageprocessingplace} (12 standard grey $512\times 512$ images). In Figure~\ref{fig:denoising} are shown the results for three images: the one for which FA$\mu$ST dictionaries perform worst (``Mandrill''), the one for which they perform best (``WomanDarkHair'') and the typical behaviour (``Pirate''). Several comments are in order:
\begin{itemize}
\item First of all, it is clear that with the considered simple denoising workflow, FA$\mu$ST dictionaries perform better than DDL at strong noise levels, namely $\sigma=30$ and $\sigma=50$. This can be explained by the fact that when training patches are very noisy, DDL is prone to overfitting (learning the noise), whereas the structure of FA$\mu$STs seems to prevent it. On the other hand, for low noise levels, we pay the lack of adaptivity of FA$\mu$STs  compared to DDL. Indeed, especially for very textured images (``Mandrill'' typically), the dictionary must be very flexible in order to fit such complex training patches, so DDL performs better. \modif{FA$\mu$ST dictionaries also perform better than DCT dictionaries at high noise levels.}

\item Second, it seems that sparser FA$\mu$STs (with fewer parameters) perform better than denser ones for high noise levels. This can be explained by the fact that they are less prone to overfitting because of their lower number of parameters, implying fewer degrees of freedom. However, this is not true for low noise levels or with too few parameters, since in that case the loss of adaptivity with respect to the training patches is too important.


\end{itemize}

\begin{figure}[htbp]
\centering
\includegraphics[width=\columnwidth]{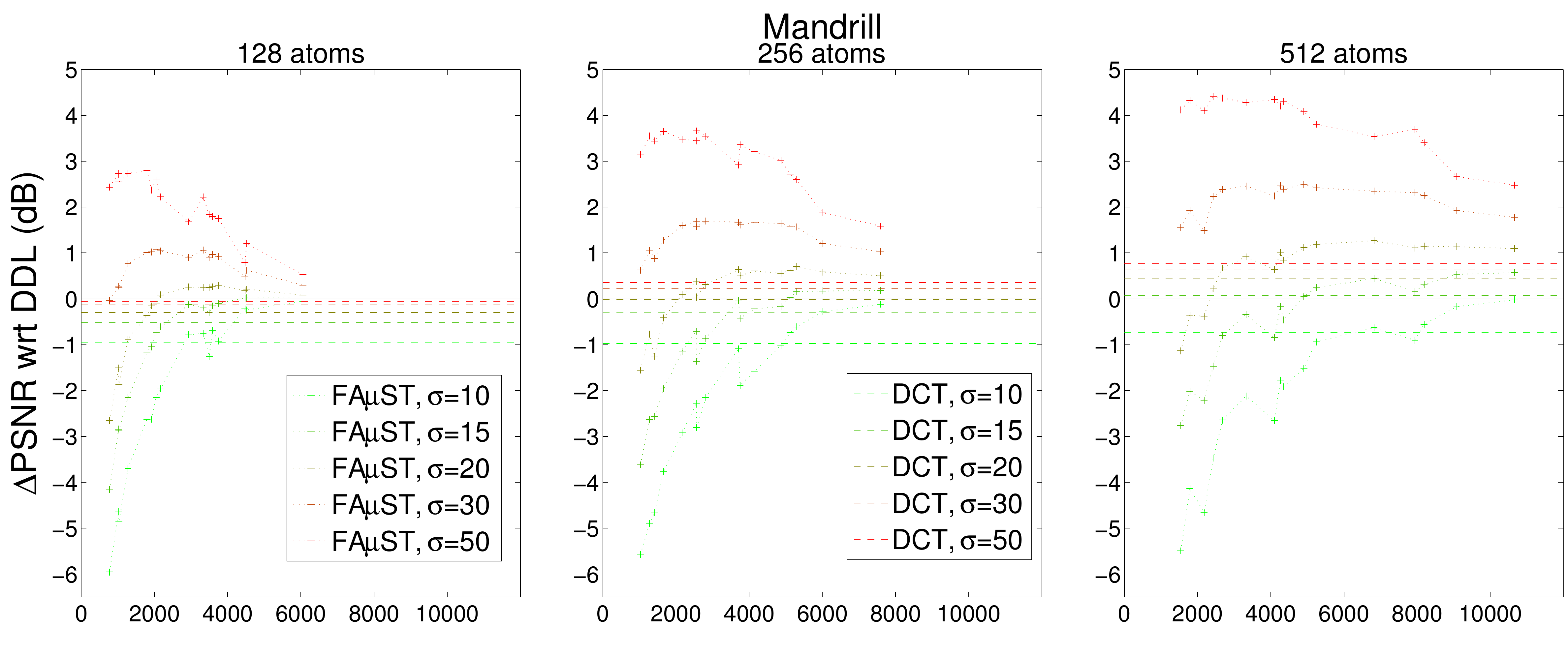}\vspace{-2mm}
\includegraphics[width=\columnwidth]{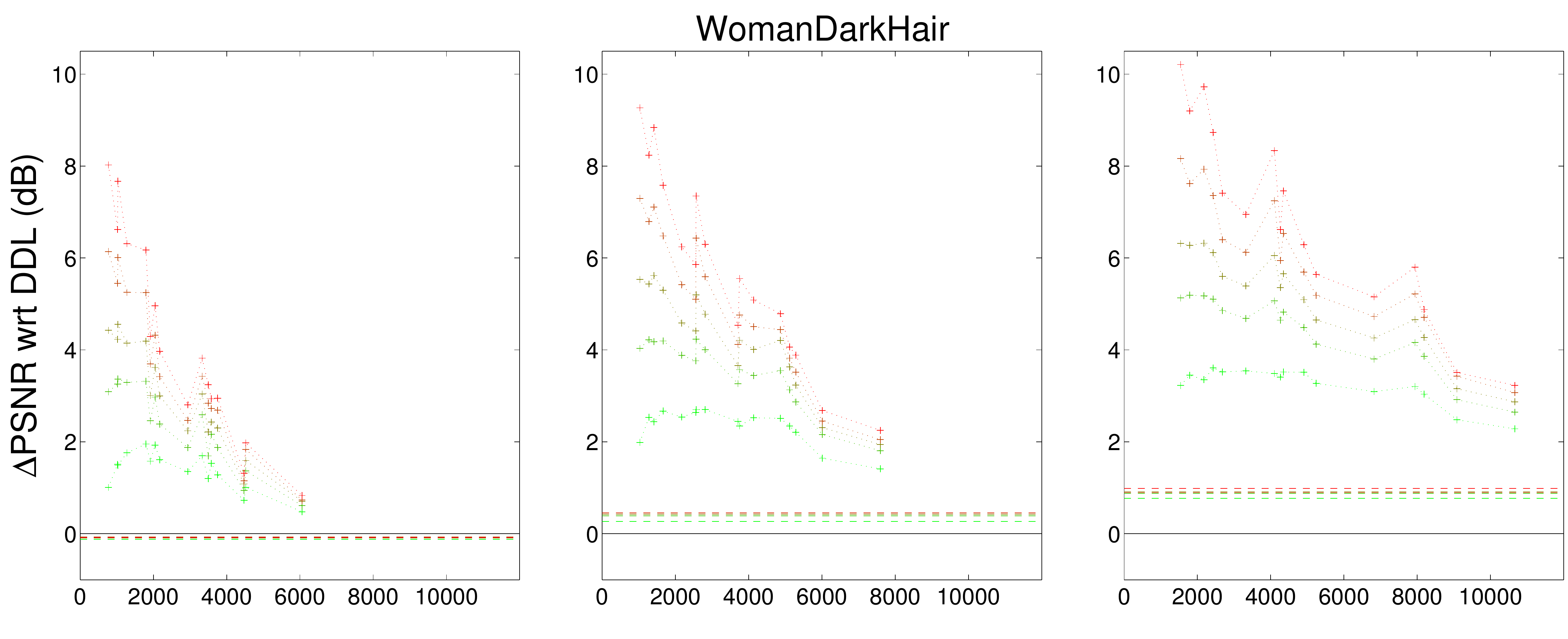}\vspace{-2mm}
\includegraphics[width=\columnwidth]{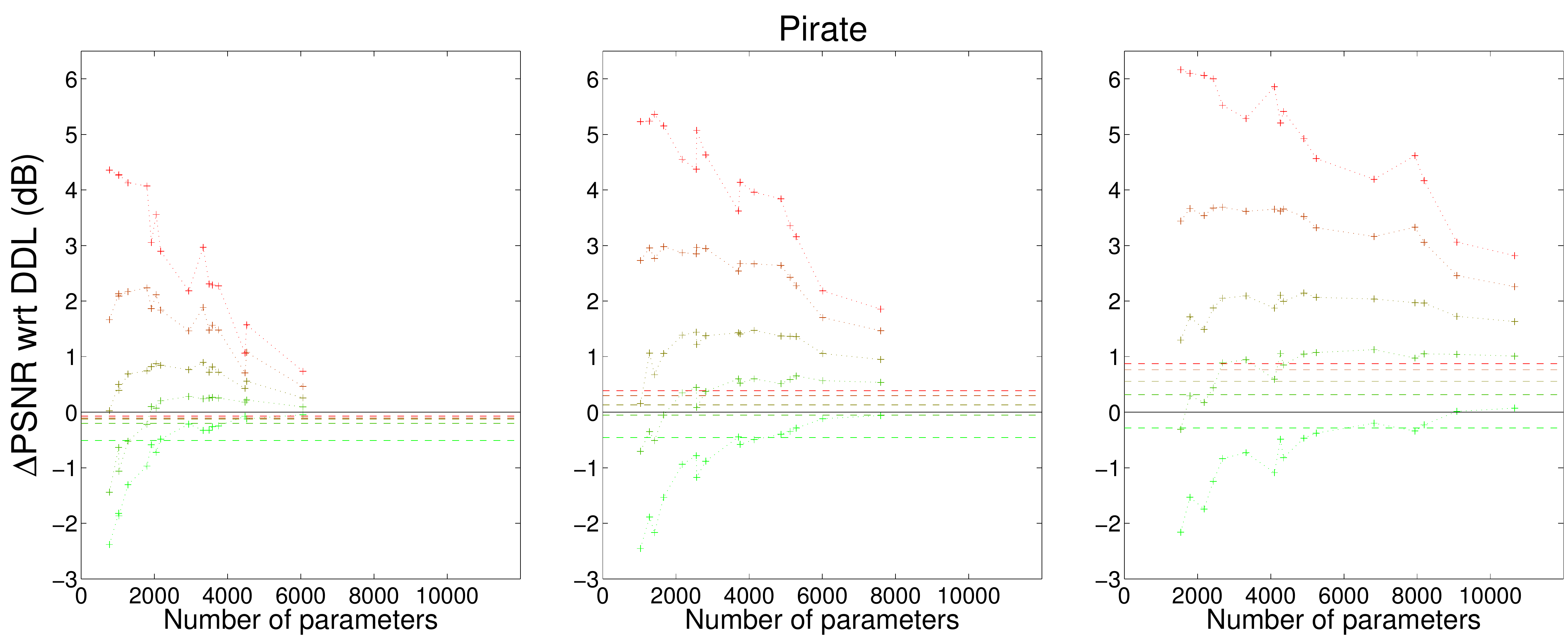}
\vspace{-0.7cm}
\caption{Denoising results. The relative performance of FA$\mu$ST dictionaries compared to DDL\modif{, and the DCT compared to DDL}
is given for several noise levels $\sigma$, for various values of $s_{\text{tot}}$ on the horizontal axis.}
\vspace{-0.6cm}
\label{fig:denoising}
\end{figure}

\subsection{Sample complexity of FA$\mu$STs}
\label{sssec:sparsedicadv}

The good performance of FA$\mu$ST dictionaries compared to dense ones observed above may be surprising, since the more constrained structure of such dictionaries (compared to dense dictionaries) may bar them from providing good approximations of the considered patches. A possible element of explanation stems from the notion of sample complexity: as evoked in section~\ref{ssec:advantages}, the statistical significance of learned multi-layer sparse operators is expected to be improved compared to that of dense operators, thanks to a reduced sample complexity. 

In the context of dictionary learning, the sample complexity indicates how many training samples $L$ should be taken in order for the empirical risk to be (with high probability) uniformly close to its expectation \cite{Vainsencher2011,Maurer2010}. In \cite{Gribonval2015}, a general bound on the deviation between the empirical risk and its expectation is provided, which is proportional to the covering dimension of the dictionary class. 

For dense dictionaries the covering dimension is known to be $\mathcal{O}(mn)$ \cite{Vainsencher2011,Maurer2010,Gribonval2015}. For FA$\mu$ST dictionaries we establish in Appendix~\ref{app:covnum} the following theorem.
\begin{theorem}
\label{th:covnum}
For multi-layer sparse operators, the covering dimension is bounded by $s_{tot}$.
\end{theorem}
A consequence is that for the same number of training samples $L$, a better generalization performance is expected from FA$\mu$ST dictionaries compared to dense ones, the gain being of the order of RCG. 
This fact is likely to explain the empirical success of FA$\mu$STs compared to dense dictionaries observed in section~\ref{sssec:imgdenois} at low SNR. Of course, when $s_{tot}$ becomes to small, the limited approximation capacity of FA$\mu$ST dictionaries imposes a trade-off between approximation and generalization.

\section{Conclusion and future work}
In this paper, a novel multi-layer matrix factorization framework was introduced, which allows to approximate a given linear operator by the composition of several ones. The underlying factorization algorithm stems on recent advances in non-convex optimization and has convergence guarantees. The proposed approach consists in hierarchically factorizing the input matrix in the hope of attaining better local minima, as is done for example in deep learning. The factorization algorithm that is used is pretty general and is able to take into account various constraints. One practical constraint of interest is sparsity (of various forms), which has several interesting properties. Indeed, multi-layer sparsely factorized linear operators have several advantages over classical dense ones, such as an increased speed of manipulation, a lighter storage footpring, and a higher statistical significance when estimated on training data. 

The interest and versatility of the proposed factorization approach was demonstrated with various experiments, including a source localization one where the proposed method performs well with a greatly reduced computational cost compared to previous techniques. We performed also image denoising experiments demonstrating that the proposed method has better generalization performance than dense dictionary learning with an impact at low SNR.

In the future, several developments are expected. On the theoretical side, we envision bounds on the trade-off between  approximation quality and relative complexity. On the experimental side, new applications for FA$\mu$STs are to be explored. For example, signal processing on graphs is a relatively new discipline where computationally efficient operators can be envisioned using learned FA$\mu$ST dictionaries, or a FA$\mu$ST approximation of graph Fourier transforms. \modif{Moreover, the factorization cost being quite high, one could envision ways to approximate matrices by FA$\mu$STs without accessing the whole matrix, in order to reduce this cost. For example, one could imagine to have observations of the form $(\mathbf{x}_i,\mathbf{y}_i = \mathbf{Ax_i})$ and try to minimize a data fitting term of the form $\sum_{i=1}^n \big\Vert \mathbf{y}_i - \prod_{j=1}^J\mathbf{S}_j\mathbf{x}_i \big\Vert_2^2$.}

\appendices
\section{Projection operators}
\label{app:projop}
In this appendix are given the projection operators onto several constraint sets of interest for practical applications.
\subsection{General sparsity constraint}
Sparsity is the most obvious constraint to put on the factors for operator sparse approximations. Consider first  the following general sparsity constraint set:
\begin{equation*}
\mathcal{E} := \{\mathbf{S} \in \mathbb{R}^{p \times q} : \left\Vert \mathbf{S}_{\mathcal{H}_i} \right\Vert_0  \leq s_i \forall i \in \{1,\dots,K\}, \left\Vert \mathbf{S} \right\Vert_F  = 1 \},
\end{equation*}
where $\{\mathcal{H}_1,\dots,\mathcal{H}_K\}$ forms a partition of the index set, $s_i \in \mathbb{N}$, $\forall i \in \{1,\dots,K\}$, and $\mathbf{S}_{\mathcal{T}}$ is the matrix whose entries match those of $\mathbf{S}$ on $\mathcal{T}$ and are set to zero elsewhere.
Given some matrix $\mathbf{U} \in \mathbb{R}^{p \times q}$, we wish to compute its projection onto the set $\mathcal{E}$: 
$
P_{\mathcal{E}}(\mathbf{U})\in \underset{\mathbf{S}}{\arg\min}\{ \left\Vert \mathbf{S} - \mathbf{U} \right\Vert_F^2 : \mathbf{S} \in \mathcal{E}\}.
$

\begin{proposition}\label{prop:proxsparsity}
Projection operator formula.

\begin{equation*}
P_{\mathcal{E}}(\mathbf{U}) = \frac{\mathbf{U}_{\mathcal{I}}}{\left\Vert\mathbf{U}_{\mathcal{I}}\right\Vert_F}
\end{equation*}
where $\mathcal{I}$ is the index set corresponding to the union of the $s_i$  entries of $\mathbf{U}_{\mathcal{H}_i}$ with largest absolute value, $\forall i \in \{1,\dots,K\}$. 
\end{proposition}

\begin{proof} Let $\mathbf{S}$ be an element of $\mathcal{E}$ and $\mathcal{J}$ its support. 
We have
$
\left\Vert \mathbf{S} - \mathbf{U} \right\Vert_F^2 
  = 1 + \left\Vert \mathbf{U} \right\Vert_F^2 
  - 2\langle\text{vec}(\mathbf{U}_\mathcal{J}),\text{vec}(\mathbf{S})\rangle.
$
For a given support, the matrix $\mathbf{S} \in \mathcal{E}$ maximizing $
\langle\text{vec}(\mathbf{U}_\mathcal{J}),\text{vec}(\mathbf{S})\rangle$ is $\mathbf{S} = \mathbf{U}_\mathcal{J}/\left\Vert\mathbf{U}_{\mathcal{J}}\right\Vert_F$. 
For this matrix, 
$
 \langle\text{vec}(\mathbf{U}_\mathcal{J}),\text{vec}(\mathbf{S})\rangle = \left\Vert\mathbf{U}_{\mathcal{J}}\right\Vert_F = \sqrt{\sum\nolimits_{i=1}^K\left\Vert\mathbf{U}_{\mathcal{J}\cap \mathcal{H}_i}\right\Vert_F^2}
 $
  which is maximized if $\mathcal{J}\cap \mathcal{H}_i$ corresponds to the $s_i$ entries with largest absolute value of $\mathbf{U}$ within $\mathcal{H}_i$,  $\forall i \in \{1,\dots,K\}$. 
\end{proof}

\subsection{Sparse and piecewise constant constraints}
Given $K$ pairwise disjoint sets $\mathcal{C}_i$ indexing matrix entries, consider now the constraint set corresponding to unit norm matrices that are constant over each index set $\mathcal{C}_i$, zero outside these sets, with no more than $s$ non-zero areas. In other words:
$
\mathcal{E}_c := \{\mathbf{S} \in \mathbb{R}^{p \times q} :  \exists \tilde{\mathbf{a}} =(\tilde{a}_{i})_{i=1}^{K}, \left\Vert\tilde{\mathbf{a}}\right\Vert_0\leq s,  \mathbf{S}_{\mathcal{C}_i} = \tilde{a}_i \forall i \in \{1,\dots,K\}, \mathbf{S}_{\overline{\bigcup_i\mathcal{C}_i}} = 0, \textrm{and}\ \left\Vert \mathbf{S} \right\Vert_F  = 1 \}.
$

Define $\tilde{\mathbf{u}} := (\tilde{u}_{i})_{i=1}^{K}$ with $\tilde{u}_{i} := \sum_{(m,n)\in \mathcal{C}_i}u_{mn}$, and denote $\tilde{\mathcal{J}} \subset \{1,\ldots,K\}$ the support of $\tilde{\mathbf{a}}$.

\begin{proposition}\label{prop:piecewisesparsity}
The projection of $\mathbf{U}$ onto $\mathcal{E}_c $ is obtained with $\tilde{\mathcal{J}}$ the collection of $s$ indices $i$ yielding the highest $|\tilde{u}_i|/\sqrt{|\mathcal{C}_i|}$, $\tilde{a}_i := \tilde{u}_i / \sqrt{\sum_{i \in \tilde{\mathcal{J}}} |\mathcal{C}_i| \tilde{u}_i^2}) $ if $i \in \tilde{\mathcal{J}}$, $\tilde{a}_i :=0$ otherwise. 
\end{proposition}

\begin{proof} Let $\mathbf{S}$ be an element of $\mathcal{E}_c$, and $\tilde{\mathcal{J}} \subset \{1,\ldots,K\}$ be the support of the associated $\tilde{\mathbf{a}}$. We proceed as for the previous proposition and notice that
$ \langle\text{vec}(\mathbf{U}),\text{vec}(\mathbf{S})\rangle 
 = \sum_{i \in \tilde{\mathcal{J}}} \langle\text{vec}(\mathbf{U}_{\mathcal{C}_i}),\text{vec}(\mathbf{S})\rangle
 = \sum_{i\in \tilde{\mathcal{J}}} \tilde{u}_{i}\tilde{a}_i 
 = \langle \tilde{\mathbf{u}}_{\tilde{\mathcal{J}}}, \tilde{\mathbf{a}}\rangle.
$
By the changes of variable $\tilde{b}_i = \sqrt{|\mathcal{C}_i|}.\tilde{a}_i$ and $\tilde{v}_i = \tilde{u}_i/\sqrt{|\mathcal{C}_i|}$ we get $\langle \tilde{\mathbf{u}}_{\tilde{\mathcal{J}}},\tilde{\mathbf{a}}\rangle = \langle \tilde{\mathbf{v}}_{\tilde{\mathcal{J}}},\tilde{\mathbf{b}}\rangle$ with $\tilde{\mathbf{b}}:=(\tilde{b}_{i})_{i=1}^{K}$. Given $\tilde{\mathcal{J}}$, maximizing this scalar product under the constraint $1 = \left\Vert\mathbf{S}\right\Vert_F = \Vert\mathbf{\tilde{b}}\Vert_2$ yields $\tilde{\mathbf{b}}^{*} := \tilde{\mathbf{v}}_{\tilde{\mathcal{J}}}/\Vert \tilde{\mathbf{v}}_{\tilde{\mathcal{J}}} \Vert_2$, and $\langle \tilde{\mathbf{v}}_{\tilde{\mathcal{J}}},\tilde{\mathbf{b}}^{*}\rangle = \Vert \tilde{\mathbf{v}}_{\tilde{\mathcal{J}}} \Vert_2$. Maximizing over $\tilde{\mathcal{J}}$ is achieved by selecting the $s$ entries of $\tilde{\mathbf{v}} :=(\tilde{\mathbf{v}})_{i=1}^{K}$ with largest absolute value (Proposition~\ref{prop:proxsparsity}). Going back to the original variables gives the result.   
\end{proof}

\section{Lipschitz modulus}
\label{app:lipmod}
To estimate the Lipschitz modulus of the gradient of the smooth part of the objective we write:
\begin{equation*}
\begin{array}{ll}
&  \left\Vert \nabla_{\mathbf{S}^i_j}H(\mathbf{L},\mathbf{S}_1,\mathbf{R},\lambda^i) - \nabla_{\mathbf{S}^i_j}H(\mathbf{L},\mathbf{S}_2,\mathbf{R},\lambda^i))\right\Vert_F \\ 
=& (\lambda^i)^2 \left\Vert \mathbf{L}^T\mathbf{L} (\mathbf{S}_1-\mathbf{S}_2) \mathbf{R}\mathbf{R}^T\right\Vert_F\\ 
\leq& (\lambda^i)^2\left\Vert \mathbf{R} \right\Vert_2^2. \left\Vert \mathbf{L} \right\Vert_2^2 \left\Vert\mathbf{S}_1-\mathbf{S}_2 \right\Vert_F.
\end{array}
\end{equation*}

\section{Covering dimension}
\label{app:covnum}
The covering number $\mathcal{N}(\mathcal{A},\epsilon)$ of a set $\mathcal{A}$ is the minimum number of balls of radius $\epsilon$ needed to cover it.
The precise definition of covering numbers is given in \cite{Gribonval2015}. The upper-box counting dimension of the set, loosely referred to as the covering dimension in the text is $d(\mathcal{A})=\lim_{\epsilon\to 0} \frac{\log \mathcal{N}(\mathcal{A},\epsilon)}{\log 1/\epsilon}$. We are interested in the covering dimension of the set of FA$\mu$STs $\mathcal{D}_{\text{spfac}}$.
We begin with the elementary sets $\mathcal{E}_j = \{\mathbf{A} \in \mathbb{R}^{a_j\times a_{j+1}} : \left\Vert \mathbf{A}\right\Vert_0\leq s_j,\left\Vert \mathbf{A}\right\Vert_F = 1\}$. These sets can be seen as sets of sparse normalized vectors of size $a_j\times a_{j+1}$.
 This leads following \cite{Gribonval2015} (with the Frobenius norm) to:
  \begin{equation*}
  \mathcal{N}(\mathcal{E}_j,\epsilon)\leq \binom{a_j a_{j+1}}{s_j}\left(1+\frac{2}{\epsilon}\right)^{s_j}.
  \end{equation*}
Defining $\mathcal{M}:= \mathcal{E}_1 \times \ldots \times \mathcal{E}_J$ and using \cite[lemma 16]{Gribonval2015} gives $\mathcal{N}(\mathcal{M},\epsilon)\leq \prod_{j=1}^{J}\binom{a_j a_{j+1}}{s_j}(1+\frac{2}{\epsilon})^{s_j}$ (wrt to the max metric over the index $j$ using the Frobenius norm). 
Using $\sum_{j=1}^J\left\Vert x_j-y_j \right\Vert_F\leq J\underset{j}{\text{max}} \left\Vert x_j-y_j \right\Vert_F$ gives $\mathcal{N}(\mathcal{M},\epsilon)\leq \prod_{j=1}^{J}\binom{a_j a_{j+1}}{s_j}(1+\frac{2J}{\epsilon})^{s_j}$ wrt to the metric defined by  $\rho(x,y) = \sum_{j=1}^J\left\Vert x_j-y_j \right\Vert_F$. 
Defining the mapping:
\begin{equation*}
\begin{array}{rll}
\Phi: & \mathcal{M}:= \mathcal{E}_1 \times \ldots \times \mathcal{E}_J &\rightarrow  \mathcal{D}_{\text{spfac}} \\
& (\mathbf{D}_1,\mathbf{D}_2,\ldots,\mathbf{D}_J) &\mapsto \mathbf{D}_J \ldots \mathbf{D}_2\mathbf{D}_1,
\end{array}
\end{equation*}
where $\mathcal{D}_{\text{spfac}}$ is the set of FA$\mu$STs of interest,
and using the distance measures $\rho(x,y) = \sum_{j=1}^J\left\Vert x_j-y_j \right\Vert_F$ in $\mathcal{M}$ and $\rho_1(x,y) = \left\Vert x-y \right\Vert_F$ in $\mathcal{D}_{\text{spfac}}$ we get that the mapping $\Phi$ is a contraction (by induction). 
We can conclude that:
\begin{equation*}
\mathcal{N}(\mathcal{D}_{\text{spfac}},\epsilon)\leq \prod\nolimits_{j=1}^{J}\binom{a_j a_{j+1}}{s_j}\left(1+\frac{2J}{\epsilon}\right)^{s_j},
\end{equation*}
wrt the $\rho_1$ metric.
Using $\binom{n}{p}\leq \frac{n^p}{p!}$ and $n!\geq \sqrt{2\pi n} \left( \frac{n}{e}\right)^n$ yields 
$\mathcal{N}(\mathcal{D}_{\text{spfac}},\epsilon)\leq \prod_{j=1}^{J}\frac{1}{\sqrt{2\pi s_j}}\left(\frac{e}{s_j}a_j a_{j+1}(1+\frac{2J}{\epsilon})\right)^{s_j}$.  
Defining $C_j = \frac{e.a_j. a_{j+1}.(2J+1)}{s_j.\sqrt[2s_j]{2\pi s_j}}$, we have $\mathcal{N}(\mathcal{D}_{\text{spfac}},\epsilon) \leq (\frac{C}{\epsilon})^h$, with $h = \sum_{j=1}^{J}s_j$ and $C = \underset{j}{\text{max}} C_j$. We thus have $d(\mathcal{D}_{\text{spfac}})\leq h = \sum_{j=1}^{J}s_j = s_{tot}$.

\section*{Acknowledgment}
{ The authors wish to thank Fran\c{c}ois Malgouyres and Olivier Chabiron for discussions that helped in producing this work. The authors also express their gratitude to Alexandre Gramfort for providing the MEG data and contributing to \cite{Lemagoarou2015a}. Finally, the authors thank the reviewers for their valuable comments.

\ifCLASSOPTIONcaptionsoff
  \newpage
\fi



\bibliographystyle{IEEEtran}
%



%


\vfill
\begin{IEEEbiography}[{\includegraphics[width=1in,height=1.25in,clip,keepaspectratio]{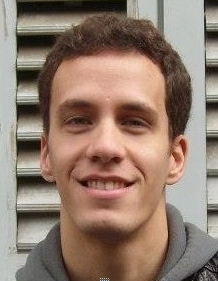}}]{Luc Le Magoarou} is a PhD student at Inria (Rennes, France). He received the M.Sc. in electrical engineering from the National Institute of Applied Sciences (INSA), Rennes, France, in 2013. His main research interests lie in signal processing and machine learning, with an emphasis on computationally efficient methods and matrix factorization.
\end{IEEEbiography}

\vfill
\begin{IEEEbiography}[{\includegraphics[width=1in,height=1.25in,clip,keepaspectratio]{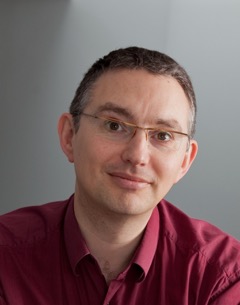}}]{R\'emi Gribonval}(FM'14)  is a Senior Researcher with Inria (Rennes, France), and the scientific leader of the PANAMA research group on sparse audio processing. A former student at  {\'E}cole Normale Sup{\'e}rieure (Paris, France), he received the Ph. D. degree in applied mathematics from Universit{\'e} de Paris-IX Dauphine (Paris, France) in 1999, and his Habilitation {\`a} Diriger des Recherches in applied mathematics from Universit{\'e} de Rennes~I (Rennes, France) in 2007. His research focuses on mathematical signal processing, machine learning, approximation theory and statistics, with an emphasis on sparse approximation, audio source separation, dictionary learning and compressed sensing. 
He founded the series of international workshops SPARS on Signal Processing with Adaptive/Sparse Representations. In 2011, he was awarded the Blaise Pascal Award in Applied Mathematics and Scientific Engineering from the SMAI by the French National Academy of Sciences, and a starting investigator grant from the European Research Council. 
\end{IEEEbiography}





\end{document}